\documentclass{article} 
\usepackage{iclr2027_conference,times}

\usepackage{amsmath,amsfonts,bm}

\def\eqref#1{equation~\ref{#1}}

\def\1{\bm{1}}

\DeclareMathAlphabet{\mathsfit}{\encodingdefault}{\sfdefault}{m}{sl}
\SetMathAlphabet{\mathsfit}{bold}{\encodingdefault}{\sfdefault}{bx}{n}

\usepackage{wrapfig}
\usepackage{url}
\usepackage{multicol}
\usepackage{amsmath,amssymb,amsthm}
\usepackage{float}
\usepackage{booktabs}
\usepackage{multirow}
\usepackage{graphicx}
\usepackage{subcaption}
\usepackage{enumitem}
\usepackage{algorithm}
\usepackage{algpseudocode}
\algrenewcommand\algorithmicrequire{\textbf{Require:}}
\algrenewcommand\algorithmicensure{\textbf{Ensure:}}
\usepackage{tikz}
\usetikzlibrary{arrows.meta,positioning,fit,calc}
\usepackage{hyperref}

\usepackage{colortbl}
\definecolor{catBase}{HTML}{F5F5F5}
\definecolor{catSearch}{HTML}{FFE6CC}
\definecolor{catEvid}{HTML}{D5E8D4}
\definecolor{catFine}{HTML}{DAE8FC}
\definecolor{catConfess}{HTML}{E1D5E7}

\newcommand{\hdr}[1]{{\fontsize{7.5pt}{9pt}\selectfont\bfseries #1}}
\graphicspath{{figures/}{../figures/}}
\newtheorem{proposition}{Proposition}

\title{OmniConfess: Eliciting Token Confessions\\
to Mitigate Omni-Modal Hallucination}

\author{%
\parbox{0.98\textwidth}{%
\centering
Huiqiang Rong$^{1}$ \quad
Haoran Luo$^{2}$\thanks{Corresponding authors:
\texttt{haoran.luo@ntu.edu.sg},
\texttt{yifan\_zhu@bupt.edu.cn}.} \quad
Hui Feng$^{1}$ \quad
Zhonghong Ou$^{1}$\\[2pt]
Kaiwen Xue$^{1}$ \quad
Guoxin Zhang$^{1}$ \quad
Yifan Zhu$^{1}$\footnotemark[1]\\[3pt]
{\normalfont\small
$^{1}$Beijing University of Posts and Telecommunications \quad
$^{2}$Nanyang Technological University, Singapore\\
\texttt{\{yifan\_zhu,rhq\}@bupt.edu.cn} \quad
\texttt{haoran.luo@ntu.edu.sg}
}%
}%
}

\iclrfinalcopy 

\begin{document}

\maketitle
\lhead{arXiv preprint}
\vspace{-6mm}
\begin{abstract}
Omni-modal large language models (OmniLLMs) unify text, images, audio, and video, yet hallucinate when generation relies on the wrong evidence. Existing inference-time methods can reduce hallucinations, but rarely reveal which evidence sustains a generated commitment. We introduce \textbf{OmniConfess}, a training-free method for mitigating omni-modal hallucinations. It fixes a candidate response and re-scores it at token resolution under controlled channel-wise evidence interventions, producing a structured token-by-channel confession that reveals the response's evidential dependence. OmniConfess uses this confession to preserve grounded content and correct commitments driven by irrelevant or contradictory evidence. To evaluate OmniConfess, we construct \textbf{OmniHalluBench}, a 3,540-example benchmark built from six datasets spanning text, image, audio, and video settings and both judgment and free-form generation. Experiments show that OmniConfess mitigates hallucinations across heterogeneous modality and task settings. Our code and benchmark are publicly available\footnote{\url{https://github.com/RongHuiQiang/OmniConfess}}.
\end{abstract}

\vspace{-2mm}
\section{Introduction}

\begin{wrapfigure}{r}{0.51\textwidth}
\vspace{-4.5mm}
\centering
\includegraphics[width=7cm]{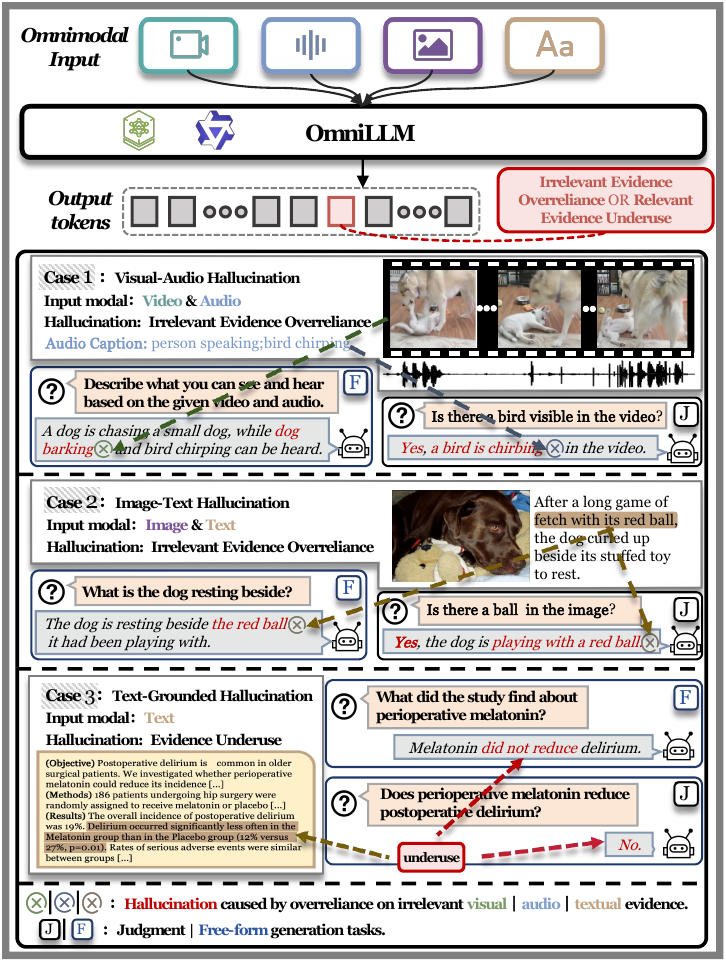}
\vspace{-2mm}
\caption{Examples of OmniLLM hallucination.}
\label{F1}
\vspace{-4mm}
\end{wrapfigure}

\label{sec:intro}

Omni-modal large language models (OmniLLMs) unify text, images, audio, and video in a single generative model, enabling reasoning over heterogeneous evidence. However, this integration can cause omni-modal hallucination when the model relies on irrelevant evidence or fails to sufficiently use relevant evidence \citep{leng2025cmm,sungbin2025avhbench} (see Figure~\ref{F1}). Hallucination control therefore requires assessing both the correctness of the final output and the evidence supporting each generated commitment.

Existing hallucination-control methods include retrieval
augmentation \citep{qu2024ara,xia2024rule}, alignment
\citep{yu2024rlhfv,sun2024llavarlhf}, and fine-tuning
\citep{chaubey2026moddpo}, often requiring additional training
or external resources. Training-free inference-time methods
instead follow three paradigms, as illustrated in
Figure~\ref{F2}: search-based methods explore or verify candidate
generations \citep{yao2023tree,xu2025phidecoding};
evidence-contrast methods steer decoding through evidence-induced
distribution shifts \citep{leng2024vcd,jung2025avcd}; and
fine-grained methods exploit token-level signals to localize
hallucinations \citep{tokenguard2026,ouyang2026taf}.

However, three main challenges remain when applied to hallucination control in OmniLLMs.
\textbf{(i) Trajectory variation.} Regeneration after evidence intervention alters the autoregressive context, confounding evidence effects with sequence differences \citep{li2025tokenactivation}.
\textbf{(ii) Entangled cross-channel dependence.} OmniLLMs jointly process text, images, audio, and video, obscuring channel contributions to individual tokens \citep{parcalabescu2023mmshap}.
\textbf{(iii) Insufficient token-level resolution.} Evidential dependence varies across tokens within each channel, making global or modality-level scores insufficient to identify revision targets \citep{chen2026mllms}.
Reliable characterization thus requires token-level channel-wise comparisons along a frozen candidate sequence.

To address these challenges, we propose \texttt{\textbf{OmniConfess}}, an inference-time hallucination correction method for OmniLLMs, with three key designs:
\textbf{(1) Commitment Anchoring.} We generate and freeze a candidate response to eliminate trajectory variation across evidence interventions.
\textbf{(2) Evidence Interrogation.} We intervene on individual evidence channels and re-score the frozen candidate to isolate token-level, channel-wise dependence.
\textbf{(3) Confession-Guided Correction.} We organize these measurements into a token-by-channel confession to recalibrate judgment scores or locate and revise erroneous free-form spans.Figure~\ref{F2} contrasts the three existing inference-time paradigms with the commitment-centered workflow of OmniConfess.

\begin{figure}[t]
  \centering
  \includegraphics[width=\linewidth]{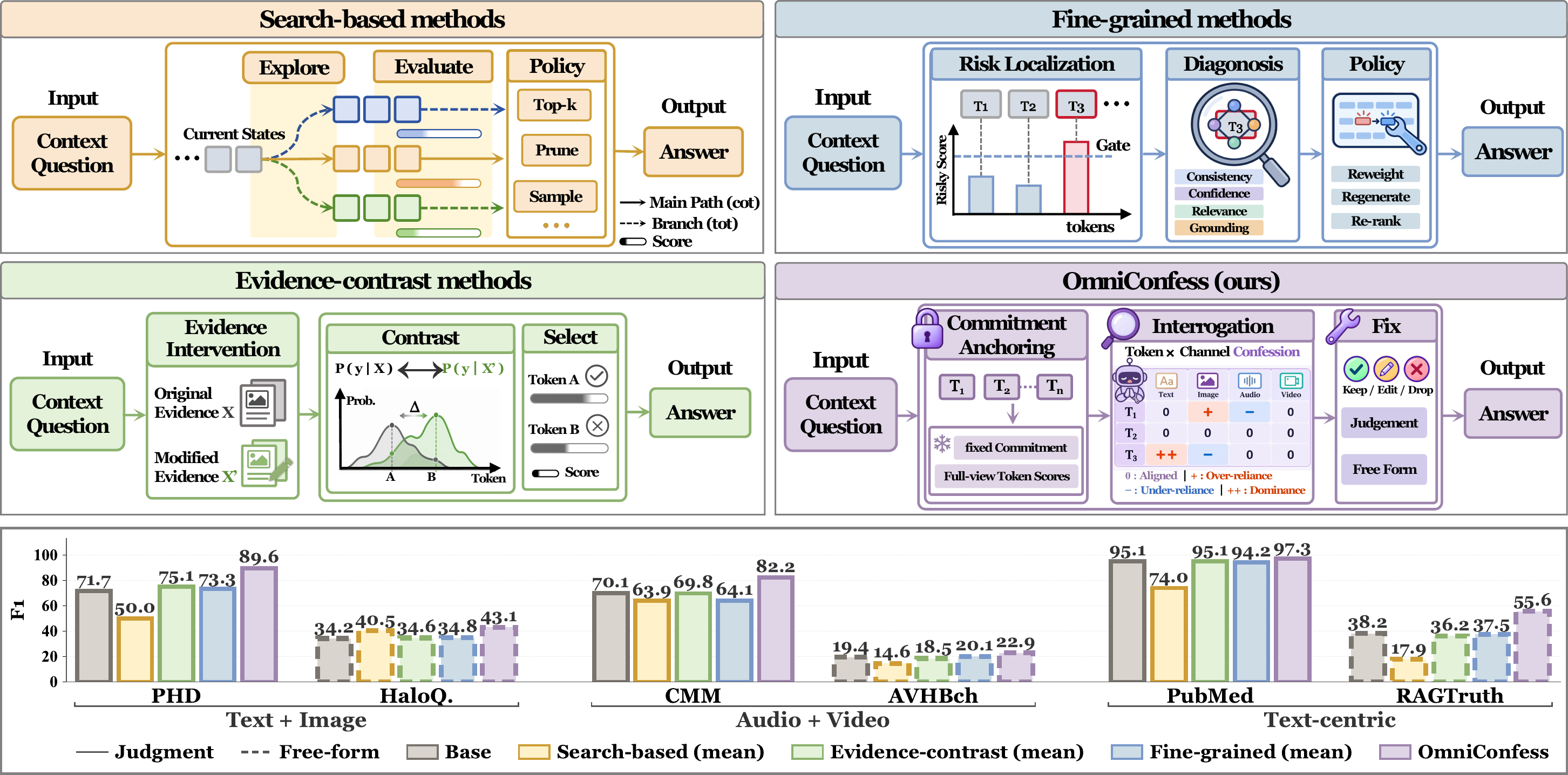}
  \caption{Comparison of three existing inference-time paradigms and OmniConfess, with category-mean F1 across six datasets using Qwen2.5-Omni-7B.}
  \label{F2}
\vspace{-2mm}
\end{figure}

To evaluate OmniConfess, we construct \texttt{\textbf{OmniHalluBench}} with 3,540 hard examples from six public datasets, spanning text, images, audio, and video across judgment and free-form generation tasks, including cross-modal conflicts and text-grounded hallucinations. OmniConfess outperforms the strongest relevant baseline across all six datasets on the primary backbone. Ablations further validate the importance of the confession's token-specific and channel-specific structure.

\vspace{-2mm}
\section{Related Work}
\vspace{-2mm}
\label{sec:related}

\paragraph{Omni-Modal Hallucination.}
As foundation models evolve from vision-language systems toward unified processing of text, images, audio, and video \citep{xu2025qwen25omni,xu2025qwen3omni,nvidia2026nemotron}, hallucination research has expanded from visual grounding failures to encompass cross-channel interference in omni-modal settings \citep{li2023pope,guan2024hallusionbench,sungbin2025avhbench,leng2025cmm}.
Recent studies on multimodal hallucination examine modality contributions to generation, differences in cross-modal grounding, and evidential attribution of generated content \citep{parcalabescu2023mmshap,kim2026omhbench,yan2026omnitrace}.

\paragraph{Inference-Time Hallucination Mitigation.}
Existing inference-time methods mainly follow three paradigms. Search-based methods explore, verify, or select candidates for reliability \citep{wang2023selfconsistency,yao2023tree,xu2025phidecoding}. Evidence-contrast methods exploit evidence-induced shifts for decoding, from visual contrast to audio-visual and modality-adaptive decoding \citep{shi2024trusting,leng2024vcd,favero2024m3id,jung2025avcd,chung2026mad,omnihallucL2026}. Fine-grained methods mitigate hallucinations via token-level or internal signals, including attention patterns, risk, and cross-modal representations \citep{huang2024opera,tokenguard2026,ouyang2026taf,li2026hb}. In this work, we introduce OmniConfess to explicitly characterize token-level, channel-wise evidential dependencies in OmniLLMs for hallucination correction.

\vspace{-2mm}
\section{Preliminary}
\label{sec:problem}
\vspace{-2mm}

We examine three properties of evidential dependence in
hallucinated OmniLLM outputs: front-loaded dependence,
cross-channel heterogeneity, and token-level concentration
(Figure~\ref{F3}).

For dataset $d$, let $\mathcal H_d$ denote its hallucinated
samples. Each $h\in\mathcal H_d$ has input $X_h$,
evidence channels $\mathcal M_h$, and fixed candidate
$y_h=(y_{h,1},\ldots,y_{h,n_h})$ of length $n_h$.
Let $\theta$ denote the OmniLLM parameters.
Token $i$'s dependence on channel $m\in\mathcal M_h$ is
\vspace{-0.5mm}
\begin{equation}
\delta_{h,i}^{(m)}
=
\log
\frac{
  p_\theta(y_{h,i}\mid y_{h,<i},X_h)
}{
  p_\theta(y_{h,i}\mid y_{h,<i},X_h^{-m})
},
\label{eq:char_delta}
\end{equation}
\vspace{-0.5mm}
where $p_\theta$ is the OmniLLM conditional token
distribution, $y_{h,<i}$ is the fixed prefix, and
$X_h^{-m}$ removes channel $m$. The overall token
dependence magnitude is
$u_{h,i}=\sum_{m\in\mathcal M_h}|\delta_{h,i}^{(m)}|$.

\begin{figure}[t]
  \centering
  \includegraphics[width=\linewidth]{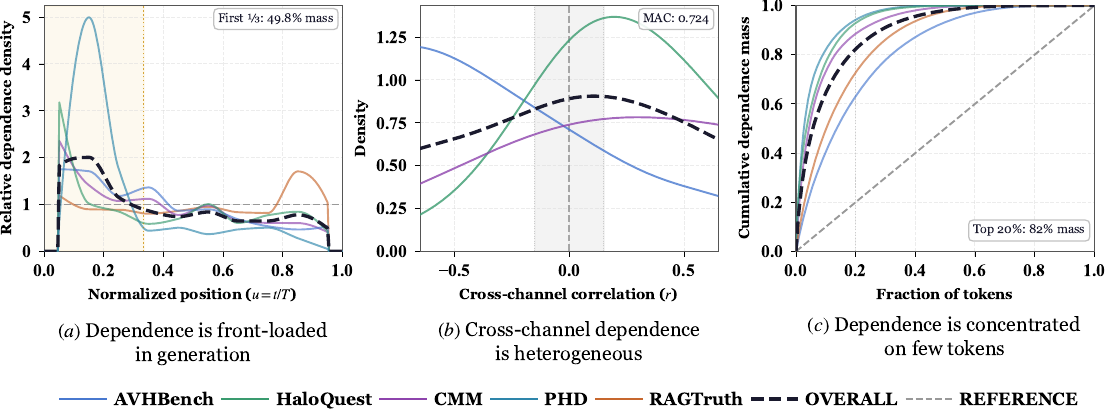}
  \vspace{-5mm}
  \caption{Three structural properties of evidential dependence in
hallucinated OmniLLM outputs.}
  \label{F3}
  \vspace{-1mm}
\end{figure}

\noindent\textbf{Front-Loaded Dependence.}
We partition normalized token positions $i/n_h$
into $B$ equal intervals $I_b=((b-1)/B,b/B]$.
The relative dependence in interval $b$ is
\begin{equation}
D_d(b)
=
\frac{1}{|\mathcal H_d|}
\sum_{h\in\mathcal H_d}
\frac{
  \sum_{i:\,i/n_h\in I_b}u_{h,i}
}{
  B^{-1}\sum_{i=1}^{n_h}u_{h,i}
},
\qquad b=1,\ldots,B.
\label{eq:char_pos}
\end{equation}
Figure~\ref{F3}(a) shows the first third holds
$49.8\%$ of dependence mass. Early token changes
alter later contexts, motivating anchoring to
separate evidence and trajectory effects.

\noindent\textbf{Cross-Channel Heterogeneity.}
We measure coupling between distinct channels
$m,m'\in\mathcal M_h$ through the correlation
of their token-level dependence trajectories:
\begin{equation}
r_h^{(m,m')}
=
\frac{
  \langle
    \widetilde{\mathbf d}_h^{(m)},
    \widetilde{\mathbf d}_h^{(m')}
  \rangle
}{
  \|\widetilde{\mathbf d}_h^{(m)}\|_2
  \|\widetilde{\mathbf d}_h^{(m')}\|_2
},
\label{eq:char_corr}
\end{equation}
where
$\widetilde{\mathbf d}_h^{(m)}
=(\delta_{h,i}^{(m)}-\bar{\delta}_h^{(m)})_{i=1}^{n_h}$
and $\bar{\delta}_h^{(m)}$ is the mean dependence.
Figure~\ref{F3}(b) shows coupled channels with
heterogeneous directions, motivating
channel-specific attribution.

\noindent\textbf{Token-Level Concentration.}
The mass captured by
the top $q$ fraction of tokens is
\begin{equation}
C_d(q)
=
\frac{1}{|\mathcal H_d|}
\sum_{h\in\mathcal H_d}
\frac{
  \sum_{j=1}^{\lceil qn_h\rceil}u_{h,(j)}^\downarrow
}{
  \sum_{i=1}^{n_h}u_{h,i}
},
\qquad q\in[0,1].
\label{eq:char_conc}
\end{equation}
With \(u_{h,(j)}^\downarrow\) denoting the $j$-th largest token dependence magnitude. Figure~\ref{F3}(c) shows that the top $20\%$ of tokens capture $82\%$ of the dependence mass on average, motivating token-level localization.

Together, these findings motivate OmniConfess's
fixed-candidate anchoring, channel-wise intervention,
and token-level rescoring for reliable evidence attribution.

\vspace{-2mm}
\section{Methodology: OmniConfess}
\vspace{-2mm}

\label{sec:method}
\begin{figure}[t]
  \centering
  \includegraphics[width=\linewidth]{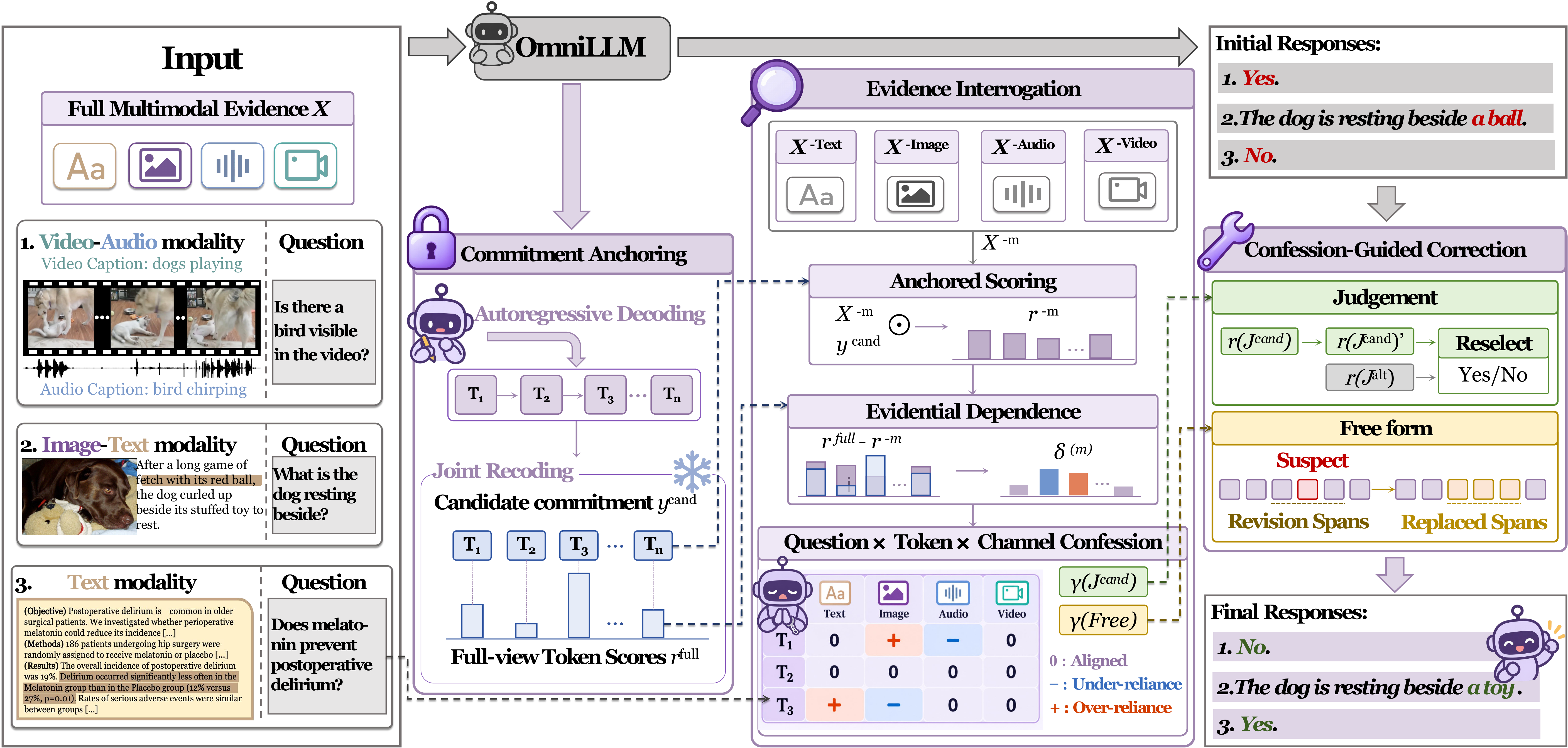}
  \vspace{-5mm}
  \caption{Overview of the OmniConfess framework: commitment anchoring, token-level evidence interrogation, and confession-guided correction for reliable omni-modal outputs.}
  \label{F4}
  \vspace{-1mm}
\end{figure}

We introduce OmniConfess
(Figure~\ref{F4}), comprising Commitment Anchoring,
Evidence Interrogation, and Confession-Guided Correction
for training-free inference-time hallucination mitigation.

\vspace{-2mm}
\subsection{Prompting Setup}
\vspace{-1.5mm}

Before the OmniConfess pipeline, we combine evidence-grounded
prompting with task-specific constraints. Prompts enforce reliance
on provided evidence and appropriate output formats, while
task-specific instructions require audible-only judgments in
\textbf{CMM}, ``Yes/No/Maybe'' in \textbf{PubMedQA},
acknowledgment of absent visual details in \textbf{HaloQuest},
and separate audio-visual observations in \textbf{AVHBench}.Full
generation and correction prompts are provided
in Appendix~\ref{app:prompts}.

\vspace{-2mm}
\subsection{Commitment Anchoring}
\vspace{-1.5mm}

For full multimodal evidence \(\mathbf X=(X^{(1)},\ldots,X^{(M)})\), OmniConfess records selected-token log-probabilities during autoregressive decoding, obtaining the candidate and full-view scores:
\begin{equation}
\left(y_{1:n}^{\mathrm{cand}},\mathbf r_{\theta}^{\mathrm{full}}\right)\triangleq\left(\operatorname{Decode}_{\theta}(\mathbf X),\left[\log p_{\theta}\!\left(y_i^{\mathrm{cand}}\mid y_{<i}^{\mathrm{cand}},\mathbf X\right)\right]_{i=1}^{n}\right)\in\mathcal V_{\mathrm{tok}}^{\,n}\times\mathbb R^n .
\label{eq:commitment_generation}
\end{equation}
where \(M\) counts evidence channels \(X^{(m)}\),
\(\theta\) denotes OmniLLM parameters, \(p_\theta\)
its token distribution, \(\operatorname{Decode}_\theta\)
autoregressive decoding, \(\mathcal V_{\mathrm{tok}}\)
the vocabulary, \(n\) candidate length, and \(i\)
token position. The candidate
\(y_{1:n}^{\mathrm{cand}}\equiv y^{\mathrm{cand}}\)
has selected token \(y_i^{\mathrm{cand}}\),
prefix \(y_{<i}^{\mathrm{cand}}\), and full-view log-probabilities \(\mathbf r_\theta^{\mathrm{full}}\).

The candidate and its prefixes are frozen during interrogation.
OmniConfess retains only \(y^{\mathrm{cand}}\),
\(\mathbf r_{\theta}^{\mathrm{full}}\), and fixed alternatives
\(a_i\) with scores: the opposing class for judgment
and the top nonselected token for free-form generation.
Under altered evidence \(\mathbf X'\), it computes anchored scores:
\begin{equation}
\mathbf r_{\theta}\!\left(\mathbf X';y^{\mathrm{cand}}\right)\triangleq\left[\log p_{\theta}\!\left(y_i^{\mathrm{cand}}\mid y_{<i}^{\mathrm{cand}},\mathbf X'\right)\right]_{i=1}^{n}\in\mathbb R^n .
\label{eq:anchored_response}
\end{equation}
where \(\mathbf r_\theta(\mathbf X';y^{\mathrm{cand}})\) denotes the frozen candidate's token-score vector under altered evidence.

\begin{proposition}
Motivated by Figure~\ref{F3}(a),
commitment anchoring eliminates trajectory variation from
token-level comparisons across evidence conditions.
\end{proposition}

\vspace{-2mm}
\noindent\textit{Proof.}
We provide experimental results in Sections~\ref{sec:ablation} and proofs in Appendix~\ref{app:anchoring_proof}.

\vspace{-2mm}
\subsection{Evidence Interrogation}
\label{sec:interrogation}
\vspace{-1.5mm}

OmniConfess interrogates frozen commitments
through channel-wise evidence removal.

\noindent\textbf{Anchored Scoring.}
OmniConfess reuses full-view scores and sequentially
rescores the candidate and alternatives under
\(\mathbf X^{-m}\), with channel \(m\) removed.
The candidate-to-alternative margin is
\begin{equation}
F_i(\mathbf Z)\triangleq
\log\frac{
p_\theta(y_i^{\mathrm{cand}}\mid y_{<i}^{\mathrm{cand}},\mathbf Z)
}{
p_\theta(a_i\mid y_{<i}^{\mathrm{cand}},\mathbf Z)
},
\quad
\mathbf Z\in\{\mathbf X\}\cup\{\mathbf X^{-m}\}_{m=1}^{M}.
\label{eq:decision_margin}
\end{equation}
\noindent\textbf{Evidential Dependence.}
For \(v\in\{y_i^{\mathrm{cand}},a_i\}\),
define token dependence and margin contribution:
\begin{equation}
\begin{aligned}
\delta_i^{(m)}(v)
&\triangleq
\log\frac{
p_\theta(v\mid y_{<i}^{\mathrm{cand}},\mathbf X)
}{
p_\theta(v\mid y_{<i}^{\mathrm{cand}},\mathbf X^{-m})
},\\[1pt]
\kappa_i^{(m)}
&\triangleq
F_i(\mathbf X)-F_i(\mathbf X^{-m}).
\end{aligned}
\label{eq:channel_dependence}
\end{equation}
\noindent\textbf{Question \(\times\) Token \(\times\) Channel Confession.}
Task rules define channel relevance
\(r_m(Q,T)\in\{0,1\}\) for question \(Q\)
and task type \(T\), with
\(\mathcal R=\{m:r_m(Q,T)=1\}\) and
\(\mathcal I=\{1,\ldots,M\}\setminus\mathcal R\).
Relevance and contribution thresholds assign
diagnostic states. The structured confession is
\begin{equation}
\mathbf C_\theta(Q,T)\triangleq
\left[
\left(
\delta_i^{(m)}(y_i^{\mathrm{cand}}),
\kappa_i^{(m)},
r_m(Q,T),
\ell_i^{(m)}
\right)
\right]_{\substack{i=1,\ldots,n\\m=1,\ldots,M}}.
\label{eq:structured_confession}
\end{equation}
where \(\ell_i^{(m)}\in\{0,-,+,++\}\)
denotes aligned, under-reliant, over-reliant,
or dominant dependence, respectively.
For retained channels
\(\mathcal S\subseteq\{1,\ldots,M\}\),
the additive surrogate is
\begin{equation}
\widehat F_i(\mathcal S)
\triangleq F_i(\mathbf X)-\sum_{m\notin\mathcal S}\kappa_i^{(m)}
=\widehat F_i(\varnothing)+\sum_{m\in\mathcal S}\kappa_i^{(m)}.
\label{eq:evidence_decomposition}
\end{equation}
where \(\widehat F_i\) is the surrogate margin.
Set \(P_i^\varnothing=\widehat F_i(\varnothing)\),
\(E_i^\mathcal R=\sum_{m\in\mathcal R}\kappa_i^{(m)}\),
and \(E_i^\mathcal I=\sum_{m\in\mathcal I}\kappa_i^{(m)}\),
denoting residual preference, relevant evidence,
and irrelevant evidence, respectively.

\noindent\textbf{Task-specific Confession.}
The resulting judgment correction and free-form
token signals are
\begin{equation}
\begin{aligned}
\gamma(J^{\mathrm{cand}})
&\triangleq E_d^\mathcal R-F_d(\mathbf X)
=-P_d^\varnothing-E_d^\mathcal I,\\[1pt]
\gamma_i^{\mathrm{free}}
&\triangleq
\left(
E_i^\mathcal R,\;
\bigl[[P_i^\varnothing]_+-[E_i^\mathcal R]_+\bigr]_+,\;
\sum_{m\notin\mathcal R}[\kappa_i^{(m)}]_+
\right).
\end{aligned}
\label{eq:task_confession}
\end{equation}
where \(J^{\mathrm{cand}}\) is the candidate class,
\(d\) its decision position, and
\([x]_+=\max(x,0)\).

\begin{proposition}
Motivated by Figure~\ref{F3}(b), the channel-wise surrogate
preserves observed margins and yields task-specific evidence signals.
\end{proposition}

\vspace{-1mm}
\noindent\textit{Proof.}
We provide experimental results in Sections~\ref{sec:ablation} and proofs in Appendix~\ref{app:interrogation_proof}.

\vspace{-1mm}
\subsection{Confession-Guided Correction}
\label{sec:correction}
\vspace{-1mm}

OmniConfess applies task-specific signals to judgment
recalibration and local free-form correction.

\noindent\textbf{Judgment Commitment Correction.}
For binary judgment, the candidate and its fixed
alternative form \(\mathcal Y=\{J^{\mathrm{cand}},a_d\}\).
OmniConfess adjusts only the candidate-class score:
\begin{equation}
\widetilde z_d(c)
\triangleq z_d(c)
+\mathbb I[c=J^{\mathrm{cand}}]\gamma(J^{\mathrm{cand}}),
\quad c\in\mathcal Y.
\label{eq:judgment_score_correction}
\end{equation}
where \(z_d(c)=\log p_\theta(c\mid
y_{<d}^{\mathrm{cand}},\mathbf X)\) is the
full-view class score and \(\mathbb I[\cdot]\)
the indicator function. The final judgment is
\(J^{\mathrm{final}}\triangleq
\arg\max_{c\in\mathcal Y}\widetilde z_d(c)\),
retaining the candidate in case of a tie.

\noindent\textbf{Free-Form Local Repair.}
For free-form generation, OmniConfess combines the
three confession signals into token-level revision risk:
\begin{equation}
\rho_i\triangleq
[-\gamma_{i,1}^{\mathrm{free}}]_+
+\gamma_{i,2}^{\mathrm{free}}
+\gamma_{i,3}^{\mathrm{free}},
\quad i=1,\ldots,n.
\label{eq:token_revision_score}
\end{equation}
where \(\gamma_{i,k}^{\mathrm{free}}\) denotes
component \(k\) of the free-form confession.
Tokens exceeding threshold \(\tau\) are expanded
to complete word boundaries and merged into
contiguous revision spans:
\begin{equation}
\mathcal B_{\mathrm{rev}}\triangleq
\operatorname{Merge}\!\left(
\operatorname{Expand}
\left(\{i:\rho_i>\tau\}\right)
\right)
=\{[s_j,e_j)\}_{j=1}^{J}.
\label{eq:revision_span_selection}
\end{equation}
where \(\mathcal B_{\mathrm{rev}}\) contains
\(J\) selected spans with start \(s_j\)
and exclusive end \(e_j\).
Each span is regenerated using the original
evidence and structured confession:
\begin{equation}
\begin{aligned}
u_j
&\triangleq
\operatorname{Rev}_{\theta}
\left(\mathbf X,y^{\mathrm{cand}},
\mathbf C_\theta(Q,T);[s_j,e_j)\right),\\[1pt]
y^{\mathrm{final}}
&\triangleq
\operatorname{Replace}\left(
y^{\mathrm{cand}},
\{[s_j,e_j)\mapsto u_j\}_{j=1}^{J}
\right).
\end{aligned}
\label{eq:local_span_revision}
\end{equation}
where \(\operatorname{Rev}_\theta\) generates replacement
\(u_j\) for span \(j\),
\(\operatorname{Replace}\) preserves content outside selected spans.

\begin{proposition}
Motivated by Figure~\ref{F3}(c), confession-guided
correction yields task-relevant judgment margins
and preserves content outside selected revision spans.
\end{proposition}
\vspace{-1mm}
\noindent\textit{Proof.}
We provide experimental results in Sections~\ref{sec:ablation} and proofs in Appendix~\ref{app:correction_proof}.
\vspace{-1mm}
\begin{table*}[t!]
\centering
\caption{Main results on \textsc{OmniHalluBench} across three
backbones. Best and second-best are in bold and
underlined, respectively. Avg. excludes missing entries;
--- denotes unavailable results.}
\label{tab:fullgrid}
\scriptsize
\setlength{\tabcolsep}{3.8pt}
\renewcommand{\arraystretch}{1.22}
\begin{tabular}{l rr rr rr rr rr rr | rrr}
\toprule
& \multicolumn{6}{c}{\hdr{Judgment tasks}} & \multicolumn{6}{c}{\hdr{Free-form tasks}} & \multicolumn{3}{c}{\hdr{Avg.}} \\
\cmidrule(lr){2-7}\cmidrule(lr){8-13}\cmidrule(lr){14-16}
& \multicolumn{2}{c}{\hdr{PHD}} & \multicolumn{2}{c}{\hdr{CMM}} & \multicolumn{2}{c}{\hdr{PubMed}} & \multicolumn{2}{c}{\hdr{RAGTruth}} & \multicolumn{2}{c}{\hdr{HaloQ.}} & \multicolumn{2}{c}{\hdr{AVHBch}} & & & \\
\cmidrule(lr){2-3}\cmidrule(lr){4-5}\cmidrule(lr){6-7}\cmidrule(lr){8-9}\cmidrule(lr){10-11}\cmidrule(lr){12-13}
\hdr{Method} & \hdr{F1} & \hdr{Idx} & \hdr{F1} & \hdr{Idx} & \hdr{F1} & \hdr{Idx} & \hdr{F1} & \hdr{GAV} & \hdr{F1} & \hdr{GAV} & \hdr{F1} & \hdr{GAV} & \hdr{F1} & \hdr{Idx} & \hdr{GAV} \\
\midrule
\multicolumn{16}{c}{\textbf{\textit{Qwen2.5-Omni-7B}}} \\
\addlinespace[2pt]
\rowcolor{catBase} Base & 71.68 & 68.35 & 70.07 & 60.77 & 95.14 & 80.21 & 38.25 & 8.77 & 34.22 & 7.16 & 19.42 & 5.96 & 54.80 & 69.78 & 7.30 \\
\rowcolor{catSearch} SC & 76.43 & 73.18 & 71.96 & 71.31 & 91.25 & 68.29 & 35.38 & \underline{8.89} & 34.48 & 7.14 & 18.29 & 6.04 & 54.63 & 70.93 & 7.36 \\
\rowcolor{catSearch} GD & 21.16 & 21.73 & 57.54 & 63.12 & 57.37 & 39.35 & 11.61 & 8.59 & 40.35 & 6.57 & 9.80 & 4.36 & 32.97 & 41.40 & 6.51 \\
\rowcolor{catSearch} PD & 32.12 & 33.01 & 58.09 & 63.61 & 79.35 & 35.20 & 12.72 & 8.57 & 38.09 & 5.43 & 10.85 & 4.76 & 38.54 & 43.94 & 6.25 \\
\rowcolor{catSearch} ToT & 70.25 & 69.07 & 67.85 & \underline{71.91} & 68.07 & 43.67 & 12.08 & 8.56 & 42.09 & \underline{7.38} & 19.46 & \underline{6.20} & 46.64 & 61.55 & \underline{7.38} \\
\rowcolor{catEvid} VCD & 72.05 & 68.64 & 69.90 & 60.32 & 95.14 & 80.21 & 38.40 & 8.69 & 34.62 & 7.09 & 19.26 & 5.99 & 54.90 & 69.72 & 7.26 \\
\rowcolor{catEvid} AVCD & 73.81 & 66.62 & 69.08 & 61.93 & 94.79 & 86.94 & 33.78 & 8.40 & 33.25 & 5.28 & 19.76 & 5.89 & 54.08 & 71.83 & 6.52 \\
\rowcolor{catEvid} CAD & \underline{79.40} & \underline{73.82} & 67.37 & 53.46 & \underline{95.24} & 83.46 & 36.29 & 8.79 & 36.06 & 6.54 & 18.01 & 5.94 & \underline{55.40} & 70.25 & 7.09 \\
\rowcolor{catEvid} MAD & --- & --- & \underline{72.80} & 65.86 & --- & --- & --- & --- & --- & --- & 16.84 & 5.88 & 44.82 & 65.86 & 5.88 \\
\rowcolor{catFine} DoLa & 72.30 & 69.57 & 70.16 & 60.99 & 93.52 & 72.09 & \underline{38.74} & 8.75 & 33.62 & 7.05 & 19.15 & 6.09 & 54.58 & 67.55 & 7.30 \\
\rowcolor{catFine} OPERA & 74.33 & 65.71 & 58.13 & 62.89 & 94.91 & \underline{87.93} & 36.36 & 8.53 & 35.91 & 5.37 & \underline{21.11} & 6.15 & 53.46 & \underline{72.18} & 6.68 \\
\rowcolor{catConfess} \textbf{OmniConfess} & \textbf{89.60} & \textbf{78.46} & \textbf{82.15} & \textbf{84.01} & \textbf{97.26} & \textbf{92.38} & \textbf{55.58} & \textbf{8.95} & \textbf{43.14} & \textbf{8.65} & \textbf{22.87} & \textbf{6.38} & \textbf{65.10} & \textbf{84.94} & \textbf{7.99} \\
\midrule
\multicolumn{16}{c}{\textbf{\textit{Qwen3-Omni-30B-A3B}}} \\
\addlinespace[2pt]
\rowcolor{catBase} Base & 83.40 & 77.24 & 74.56 & 73.69 & 93.93 & 71.02 & 36.30 & 8.85 & 33.38 & 7.25 & 19.03 & 6.28 & 56.77 & 73.98 & 7.46 \\
\rowcolor{catSearch} SC & 85.19 & \underline{78.99} & \textbf{80.72} & \textbf{83.39} & 94.76 & 65.82 & 37.33 & 8.98 & 27.37 & 7.34 & 16.78 & 6.24 & 57.03 & 76.07 & \underline{7.52} \\
\rowcolor{catSearch} GD & 38.66 & 39.46 & 58.97 & 64.89 & 77.22 & 28.46 & 22.42 & 8.08 & 35.19 & 6.37 & 7.46 & 4.32 & 39.99 & 44.27 & 6.26 \\
\rowcolor{catSearch} PD & 40.59 & 41.30 & 58.50 & 63.21 & 67.85 & 34.70 & 19.90 & 8.15 & 31.71 & 6.87 & 7.03 & 3.66 & 37.60 & 46.40 & 6.23 \\
\rowcolor{catSearch} ToT & 83.51 & 78.26 & 77.54 & 79.93 & 83.38 & 31.50 & 22.14 & 8.07 & \textbf{41.50} & \underline{7.57} & 18.59 & 6.43 & 54.44 & 63.23 & 7.36 \\
\rowcolor{catEvid} VCD & 84.47 & 77.65 & 73.68 & 72.39 & 93.93 & 71.02 & 35.73 & 8.96 & 33.65 & 6.77 & 19.07 & 6.23 & 56.76 & 73.69 & 7.32 \\
\rowcolor{catEvid} AVCD & 84.54 & 77.84 & 74.32 & 73.12 & 93.40 & 70.74 & 36.19 & \underline{9.01} & 33.21 & 6.84 & 19.14 & 6.19 & 56.80 & 73.90 & 7.35 \\
\rowcolor{catEvid} CAD & \underline{85.71} & 78.00 & 78.65 & 80.16 & \underline{94.88} & \underline{80.03} & 38.46 & 8.96 & 26.76 & 7.41 & 17.74 & 6.16 & 57.03 & 79.40 & 7.51 \\
\rowcolor{catEvid} MAD & --- & --- & 78.21 & 80.84 & --- & --- & --- & --- & --- & --- & 21.49 & 6.30 & 49.85 & \underline{80.84} & 6.30 \\
\rowcolor{catFine} DoLa & 83.91 & 78.01 & 73.71 & 71.24 & 93.12 & 71.87 & 38.14 & 8.96 & 31.43 & 6.74 & 18.81 & 6.21 & 56.52 & 73.71 & 7.30 \\
\rowcolor{catFine} OPERA & 80.11 & 76.11 & 74.30 & 76.24 & 94.15 & 74.94 & \underline{41.30} & 8.94 & 37.35 & 6.95 & \textbf{21.74} & \underline{6.50} & \underline{58.16} & 75.76 & 7.46 \\
\rowcolor{catConfess} \textbf{OmniConfess} & \textbf{87.47} & \textbf{79.10} & \underline{80.24} & \underline{81.11} & \textbf{96.76} & \textbf{83.84} & \textbf{53.20} & \textbf{9.06} & \underline{39.32} & \textbf{8.14} & \underline{21.66} & \textbf{6.58} & \textbf{63.11} & \textbf{81.35} & \textbf{7.93} \\
\midrule
\multicolumn{16}{c}{\textbf{\textit{Nemotron-3-Nano-Omni-30B}}} \\
\addlinespace[2pt]
\rowcolor{catBase} Base & 83.14 & 74.98 & 76.36 & 79.46 & 78.62 & 66.94 & \underline{45.10} & --- & 16.48 & --- & 18.67 & --- & 53.06 & 73.79 & --- \\
\rowcolor{catSearch} SC & 84.68 & 76.06 & 78.82 & 80.57 & 72.62 & 58.72 & 42.10 & 9.02 & 16.37 & 6.68 & 17.39 & 5.94 & 52.00 & 71.78 & 7.21 \\
\rowcolor{catSearch} GD & \textbf{86.36} & \textbf{81.23} & 76.06 & 78.55 & 23.72 & 19.61 & 26.89 & \underline{9.11} & \underline{35.64} & \underline{8.28} & 16.63 & 5.58 & 44.22 & 53.26 & \underline{7.66} \\
\rowcolor{catSearch} PD & 82.54 & 79.05 & 77.61 & 80.00 & 24.08 & 23.72 & 26.75 & \textbf{9.12} & 30.43 & 7.85 & 17.29 & 5.97 & 43.13 & 54.38 & 7.65 \\
\rowcolor{catSearch} ToT & \underline{85.50} & \underline{80.53} & \textbf{80.56} & \textbf{82.54} & 37.93 & 21.36 & 26.48 & 9.09 & 30.40 & 7.53 & 16.47 & 5.60 & 46.22 & 61.48 & 7.41 \\
\rowcolor{catEvid} VCD & 83.80 & 74.99 & 76.83 & 79.65 & 72.20 & 55.08 & 43.21 & 9.06 & 24.32 & 6.09 & 21.19 & \underline{6.25} & 53.59 & 69.91 & 7.13 \\
\rowcolor{catEvid} AVCD & 59.03 & 53.49 & 56.94 & 40.02 & 69.84 & 63.76 & 24.92 & 9.01 & 25.21 & 6.20 & \textbf{22.22} & 6.21 & 43.03 & 52.42 & 7.14 \\
\rowcolor{catEvid} CAD & 83.04 & 74.91 & 78.71 & 80.43 & \underline{82.66} & \underline{69.71} & 44.08 & 8.96 & 32.04 & 7.03 & 19.63 & 6.24 & \underline{56.69} & \underline{75.02} & 7.41 \\
\rowcolor{catEvid} MAD & --- & --- & 73.61 & 74.70 & --- & --- & --- & --- & --- & --- & \underline{22.10} & \textbf{6.30} & 47.86 & 74.70 & 6.30 \\
\rowcolor{catFine} DoLa & 83.42 & 74.72 & 73.72 & 77.26 & 62.66 & 55.45 & 44.37 & 8.99 & 22.14 & 5.92 & 21.20 & 6.18 & 51.25 & 69.14 & 7.03 \\
\rowcolor{catFine} OPERA & 58.84 & 58.24 & 76.60 & 79.38 & 77.23 & 63.40 & 44.90 & 8.96 & 30.23 & 6.58 & 20.04 & 6.04 & 51.31 & 67.01 & 7.19 \\
\rowcolor{catConfess} \textbf{OmniConfess} & 84.66 & 75.58 & \underline{80.50} & \underline{81.95} & \textbf{90.63} & \textbf{71.65} & \textbf{51.71} & 9.12 & \textbf{40.75} & \textbf{8.57} & 21.14 & \underline{6.25} & \textbf{61.57} & \textbf{76.39} & \textbf{7.88} \\
\bottomrule
\end{tabular}
\end{table*}

\section{Experiments}
\label{sec:experiments}

We answer six research questions:
\textbf{RQ1:} Does OmniConfess outperform existing methods?
\textbf{RQ2:} Does each stage of OmniConfess work?
\textbf{RQ3--5:} How are semantic
quality, cross-backbone generalization, and inference efficiency?

\subsection{Experimental Setup}

  \noindent\textbf{Datasets.}
We construct \textsc{OmniHalluBench}, comprising
3,540 examples from six existing benchmarks.
Judgment tasks include \textbf{CMM}~\citep{leng2025cmm}
(Audio--Visual--Text), \textbf{PhD}~\citep{liu2025phd}
(Image--Text), and \textbf{PubMedQA}~\citep{jin2019pubmedqa}
(Text). Free-form tasks include
\textbf{AVHBench}~\citep{sungbin2025avhbench}
(Audio--Visual--Text),
\textbf{HaloQuest}~\citep{tong2024haloquest}
(Image--Text), and
\textbf{RAGTruth}~\citep{niu2024ragtruth}
(Text). Construction details are provided in Appendix~\ref{app:dataset}.

  \noindent\textbf{Backbones and Baselines.}
  We evaluate OmniConfess on three OmniLLMs:
 \texttt{Qwen2.5-\allowbreak{}Omni-\allowbreak{}7B}~\citep{xu2025qwen25omni}
  (default),
  \texttt{Qwen3-\allowbreak{}Omni-\allowbreak{}30B-\allowbreak{}A3B}~\citep{xu2025qwen3omni},
  and \texttt{Nemotron-3-\allowbreak{}Nano-\allowbreak{}Omni-\allowbreak{}30B}~\citep{nvidia2026nemotron}.
  Baselines include search-based methods
  (SC~\citep{wang2023selfconsistency},
  GD~\citep{xie2023selfeval},
  PD~\citep{xu2025phidecoding},
  ToT~\citep{yao2023tree}),
  evidence-contrast methods
  (VCD~\citep{leng2024vcd},
  AVCD~\citep{jung2025avcd},
  CAD~\citep{shi2024trusting},
  MAD~\citep{chung2026mad}),
  and fine-grained methods
  (DoLa~\citep{chuang2024dola},
  OPERA~\citep{huang2024opera}),
  alongside the base model.
  Further details are in Appendix~\ref{app:baselines}.

  \noindent\textbf{Evaluation Metrics.}
  Judgment tasks use F1 and
  PhD-Index (Idx)~\citep{liu2025phd};
  free-form tasks use Token-F1 and
  GAVIE-inspired accuracy (GAV) (GAV)~\citep{liu2024mitigating},
  scored by an LLM judge on a 0--10 scale
using our task-specific evaluation prompt.
  Average F1 covers six datasets, while
  average Idx and GAV each cover three.For hallucination localization, we report
span-level AUPRC.
  Metric definitions are in
  Appendix~\ref{app:evaluation}.

  \noindent\textbf{Implementation Details.}
  OmniConfess requires no additional training.
  All experiments are conducted on 4 NVIDIA A40
  GPUs.

\begin{figure*}[t]
\centering

\begin{subfigure}[t]{0.52\textwidth}
    \centering
    \vspace{0pt}

    \scriptsize
    \setlength{\tabcolsep}{5pt}
    \renewcommand{\arraystretch}{1.15}

    \resizebox{\linewidth}{!}{%
    \begin{tabular}{lcccc|ccc}
        \toprule
        \multirow{2}{*}{\textbf{Configuration}}
        & \multicolumn{2}{c}{\textbf{Judgment Avg.}}
        & \multicolumn{2}{c|}{\textbf{HaloQuest}}
        & \multicolumn{3}{c}{\textbf{Avg.}} \\
        \cmidrule(lr){2-3}
        \cmidrule(lr){4-5}
        \cmidrule(lr){6-8}
        & \textbf{F1}
        & \textbf{Idx}
        & \textbf{F1}
        & \textbf{GAV}
        & \textbf{F1}
        & \textbf{Idx}
        & \textbf{GAV} \\
        \midrule
        w/o Anchoring
        & 80.72 & 76.74 & 39.76 & 8.12
        & 67.06 & 76.74 & 8.12 \\

        w/o Interrogation
        & 77.25 & 73.42 & 37.38 & 7.83
        & 63.96 & 73.42 & 7.83 \\

        w/o Correction
        & 70.88 & 64.56 & 34.22 & 7.16
        & 58.66 & 64.56 & 7.16 \\

        \rowcolor{catConfess}
        \textbf{OmniConfess}
        & \textbf{85.88} & \textbf{81.24}
        & \textbf{43.14} & \textbf{8.65}
        & \textbf{71.63} & \textbf{81.24}
        & \textbf{8.65} \\
        \bottomrule
    \end{tabular}%
    }

    \vspace{4.5mm}
    \caption{Ablation Study}
    \label{f5a}
\end{subfigure}
\hfill
\begin{subfigure}[t]{0.46\textwidth}
    \centering
    \vspace{0pt}

    \includegraphics[width=\linewidth]{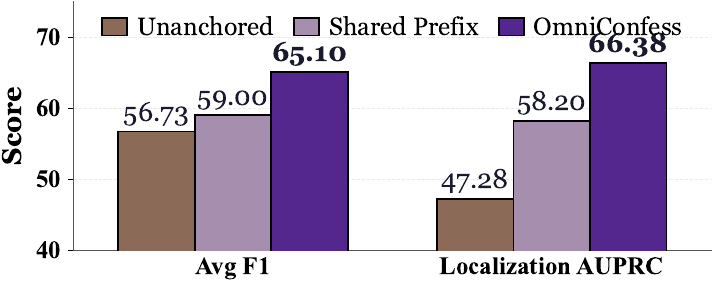}

    \caption{Anchoring Analysis}
    \label{f5b}
\end{subfigure}

\par\vspace{2mm}

\begin{subfigure}[t]{0.36\textwidth}
    \centering
    \vspace{0pt}

    \includegraphics[width=\linewidth]{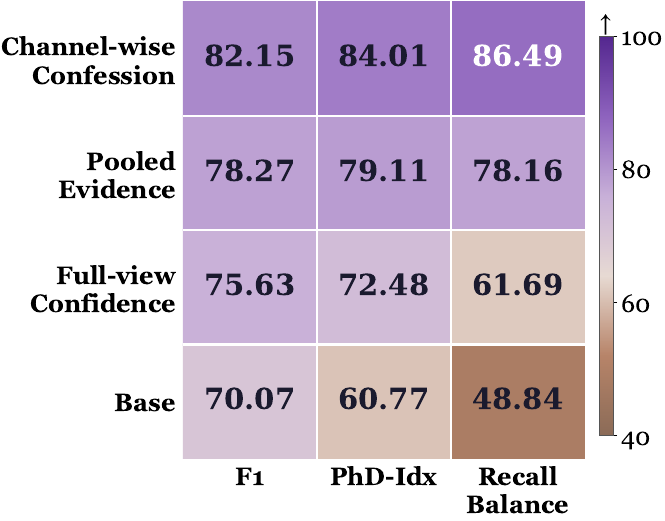}

    \vspace{1.5mm}
    \caption{Evidence Interrogation}
    \label{f5c}
\end{subfigure}
\hfill
\begin{subfigure}[t]{0.59\textwidth}
    \centering
    \vspace{0pt}

    \includegraphics[width=\linewidth]{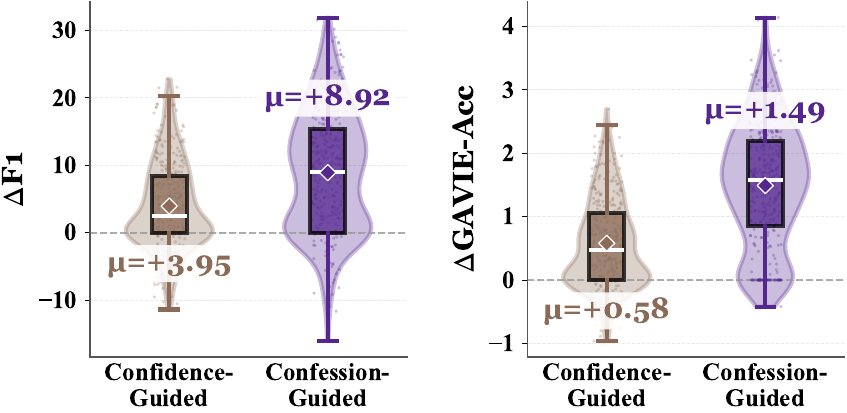}

    \vspace{1mm}
    \caption{Correction Strategies}
    \label{f5d}
\end{subfigure}

\caption{
(a) Module ablation of OmniConfess.
(b) Comparison of commitment anchoring strategies.
(c) Comparison of evidence interrogation strategies.
(d) Comparison of correction strategies.
}
\label{fig:F5}

\vspace{-1mm}
\end{figure*}
\vspace{-2mm}
\subsection{Main Results (RQ1) }
\vspace{-2mm}

\label{sec:eff}
As shown in Table~\ref{tab:fullgrid}, OmniConfess
consistently outperforms all baselines across
six datasets and all reported metrics on
Qwen2.5-Omni-7B. We further identify three
key observations.

\noindent\textbf{Hallucination Control Across
Evidence Modalities.}
OmniConfess consistently improves multimodal and text-only tasks.
On PHD and CMM, F1 exceeds the strongest baselines by 10.20 and
9.35 points, suggesting channel-wise evidence interrogation curbs
reliance on irrelevant channels. On RAGTruth, Token-F1 gains 16.84
points, suggesting confession mitigates underuse of relevant evidence.
Despite PubMedQA's strong baseline, F1 rises by 2.02 points.

\noindent\textbf{Hallucination Control Across
Generation Tasks.}
For judgment tasks, improvements in F1 and Index
indicate gains in decision accuracy and reduced
hallucination-related bias. For free-form generation,
gains in Token-F1 and GAVIE reflect gains in lexical
consistency and factual accuracy. These results show
token-level evidence correction benefits both judgment
and generation.

\noindent\textbf{Behavior of Existing Inference-Time Methods.}
Search-based methods show task sensitivity, with
search or iteration sometimes degrading performance
under multimodal evidence.
Evidence-contrast methods remain close to the base model,
as modality interventions may suppress evidence
along with distractions.
Fine-grained methods lack improvements, suggesting
token-level risk signals may not determine whether
evidence dependence should be strengthened or reduced.

\subsection{Ablation Study and Comparative Analysis (RQ2)}
\label{sec:ablation}

Figure~\ref{fig:F5} presents the ablation study and comparative analysis of OmniConfess.

\noindent\textbf{Module Ablation.}
Figure~\ref{fig:F5}\subref{f5a} ablates anchoring, evidence
interrogation, and correction. OmniConfess leads all metrics,
averaging 71.63 F1. Removing anchoring, interrogation, or correction
lowers average F1 to 67.06, 63.96, or 58.66, respectively.
All modules contribute; correction's removal has the largest effect,
underscoring the value of revision after identifying evidence-related
errors.

\noindent\textbf{Anchoring Analysis.}
Figure~\ref{fig:F5}\subref{f5b} compares the unanchored setting, shared-prefix anchoring \citep{shi2024trusting}, and OmniConfess. Shared-prefix anchoring increases average F1 from 56.73 to 59.00 and localization AUPRC from 47.28 to 58.20. OmniConfess improves them to 65.10 and 66.38, respectively. Thus, anchoring benefits both prediction and error localization, and the anchoring used by OmniConfess is more effective than a shared prefix alone.

\noindent\textbf{Evidence Interrogation Strategies.}
Figure~\ref{fig:F5}\subref{f5c} compares full-view confidence \citep{ding-etal-2025-lvlms}, pooled evidence \citep{chung2026mad}, and channel-wise confession. It leads in F1 (82.15) and PhD-Index (84.01). We define $\mathrm{Recall\ Balance}=100-|R_{\mathrm{YES}}-R_{\mathrm{NO}}|$, with recall on a 0--100 scale; higher values indicate a smaller YES--NO gap. Channel-wise confession scores 86.49, versus 78.16 for pooled evidence, 61.69 for full-view confidence, and 48.84 for the base model. This reduction helps mitigate the model's bias toward predicting YES for NO-labeled samples.

\noindent\textbf{Correction Strategies.}
Figure~\ref{fig:F5}\subref{f5d} compares revision strategies under an equal correction budget. Confession-guided revision yields larger mean gains than confidence-guided revision in F1 ($+8.92$ vs.\ $+3.95$) and GAVIE-Acc ($+1.49$ vs.\ $+0.58$). This supports using token- and channel-level evidence dependence to locate spans for correction.

\begin{figure*}[t]
    \centering

    \begin{subfigure}[t]{0.47\textwidth}
        \centering
        \vspace{0pt}
        \includegraphics[width=\linewidth]{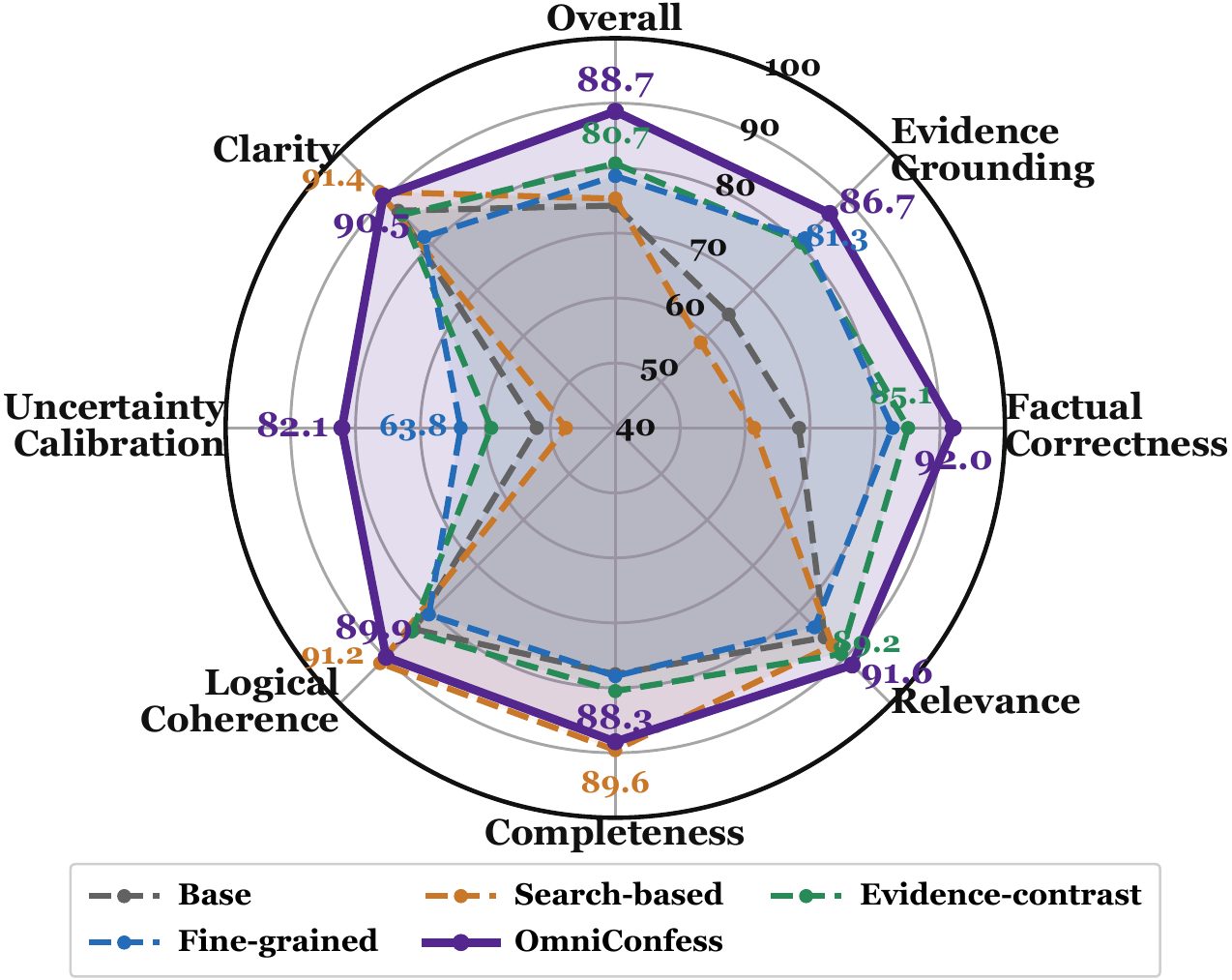}

        \vspace{1.5mm}
        \caption{Generation Quality Comparison}
        \label{f6a}
    \end{subfigure}
    \hfill
    \begin{subfigure}[t]{0.52\textwidth}
        \centering
        \vspace{0pt}
        \includegraphics[width=\linewidth]{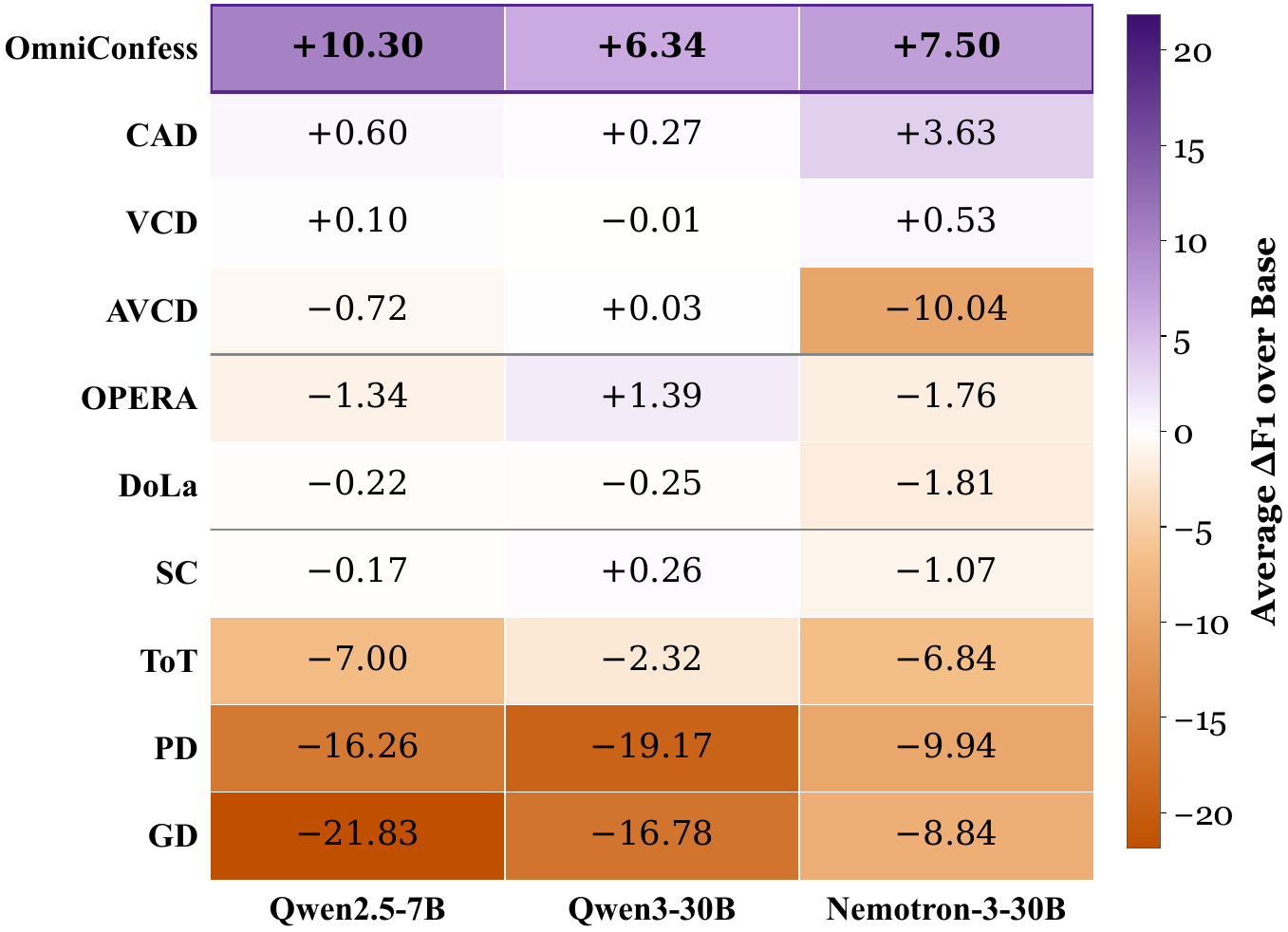}

        \vspace{1mm}
        \caption{Cross-Backbone Performance}
        \label{f6b}
    \end{subfigure}

    \caption{Generation quality and cross-backbone portability of OmniConfess.
    (a) Comparison of OmniConfess and baseline methods across
    generation quality dimensions.
    (b) Average F1 changes relative to each backbone's base model
    across three architectures.}
    \label{F6}
    \vspace{-2mm}
\end{figure*}

\vspace{-1mm}
\subsection{Analysis of OmniConfess's Impact in Generation Quality (RQ3)}
\label{sec:mechanism}
\vspace{-1mm}

To assess generation quality beyond task metrics, we use
GPT-4o~\citep{openai2024gpt4o} and human experts to cross-evaluate
seven quality dimensions and an Overall score.
Figure~\ref{F6}\subref{f6a} compares inference-time baselines
with OmniConfess. The prompt and protocol are detailed in Appendix~\ref{app:quality_prompts}.

\noindent\textbf{Baseline Analysis.}
Search-based methods retain clarity and coherence but lag in grounding
and factual correctness, indicating that refinement need not improve
evidential support. Evidence-contrast methods offer balanced quality,
whereas Fine-grained interventions strengthen grounding but sacrifice
linguistic quality. Both remain weak in uncertainty calibration.

\noindent\textbf{OmniConfess Analysis.}
OmniConfess achieves the highest Overall score (88.7), with Evidence
Grounding (86.7), Factual Correctness (92.0), and Uncertainty Calibration
(82.1) exceeding the strongest baselines by 5.4, 6.9, and 18.3 points,
respectively. Its competitive clarity and coherence, alongside markedly
improved uncertainty calibration, suggest that token- and channel-level
evidence interrogation supports reliable correction without sacrificing
response quality.

\subsection{Analysis of OmniConfess's Portability Across Backbone Models (RQ4)}
\label{sec:cross_backbone_portability}

\begin{figure}[t]
  \centering
  \captionsetup{skip=3pt}
  \includegraphics[width=\linewidth]{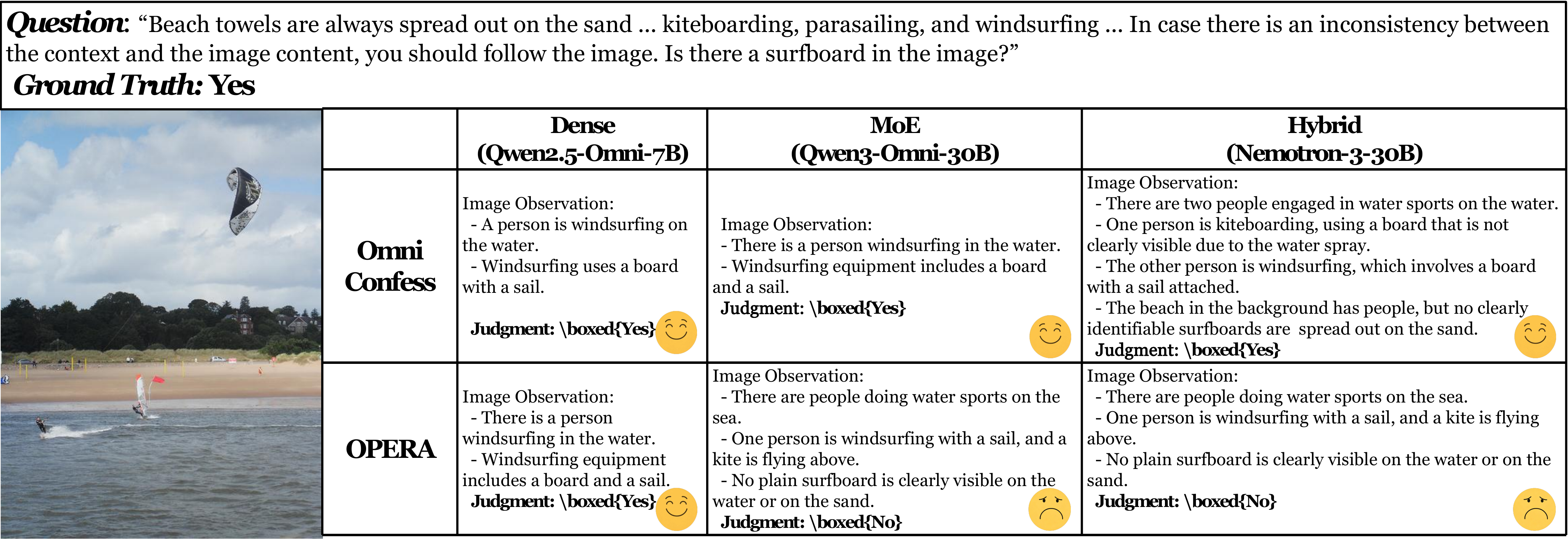}
  \caption{Case study of OmniConfess and OPERA across three backbones on PHD.}
  \label{F7}
  \vspace{-2mm}
\end{figure}

To assess architectural generalizability, we evaluate methods
on Qwen2.5-Omni-7B (Dense), Qwen3-Omni-30B-A3B (MoE), and
Nemotron-3-Nano-Omni-30B (Hybrid). Figure~\ref{F6}\subref{f6b}
shows that OmniConfess improves average F1 over their respective
base models by 10.30, 6.34, and 7.50 points.

Unlike OmniConfess, baseline methods either vary across backbones
or perform consistently poorly. AVCD, OPERA, and SC switch between
gains and losses, improving only on MoE but declining on Dense and
Hybrid, while GD, PD, and ToT reduce F1 on all three backbones.
Even CAD, which improves all three, yields smaller gains than
OmniConfess. These results highlight OmniConfess's consistent
advantage across Dense, MoE, and Hybrid architectures.

\noindent\textbf{Case Study.}
Figure~\ref{F7} shows larger backbones do not guarantee better
grounding. OPERA succeeds on Dense, but its longer reasoning on
30B backbones narrows the target category and overrides visual
evidence. Architecture-dependent reasoning may amplify multimodal
hallucinations. OmniConfess remains correct across architectures, showing its robustness.

\vspace{-1mm}
\subsection{Analysis of OmniConfess's Inference Efficiency (RQ5)}
\label{sec:efficiency}
\vspace{-1mm}

We evaluate inference time, generated tokens, and normalized
throughput in Table~\ref{tab:inference_efficiency}.

\noindent\textbf{Evidence Interrogation Overhead.}
Judgment outputs are short, yet OmniConfess takes 5.43 seconds on
PHD and 35.14 seconds on audio-visual CMM, compared with 2.23 and
20.64 seconds for BaseModel. CMM generates only 24.6 tokens,
suggesting that repeated channel-wise evidence rescoring, rather
than output length alone, accounts for much of its latency.

\noindent\textbf{Token--Latency Trade-off.}
Compared with search-based methods, OmniConfess uses about
27$\times$ and 22$\times$ fewer generated tokens on PHD and CMM,
respectively, but reduces inference time by only 6.4$\times$ and
1.1$\times$. On AVHBench, generated tokens fall from 569.5 to 164.1,
while time falls from 36.15 to 7.93 seconds. Anchoring a candidate
limits generation cost, but multimodal evidence interrogation still
incurs substantial latency on some tasks.

\newcommand{\effcell}[3]{%
  #1\,\textperiodcentered\,#2\,\textperiodcentered\,#3%
}

\begin{table*}[t]
\centering
\captionsetup{skip=3pt}
\caption{Time, output token consumption, and normalized throughput (tokens/sec).}
\label{tab:inference_efficiency}

\scriptsize
\setlength{\tabcolsep}{5pt}
\renewcommand{\arraystretch}{1.15}

\resizebox{\textwidth}{!}{%
\begin{tabular}{lcccccc}
\toprule
\textbf{Category}
& \textbf{PHD}
& \textbf{PubMedQA}
& \textbf{CMM}
& \textbf{RAGTruth}
& \textbf{HaloQuest}
& \textbf{AVHBench} \\

& Time\,\textperiodcentered\,Tokens\,\textperiodcentered\,Tok/s
& Time\,\textperiodcentered\,Tokens\,\textperiodcentered\,Tok/s
& Time\,\textperiodcentered\,Tokens\,\textperiodcentered\,Tok/s
& Time\,\textperiodcentered\,Tokens\,\textperiodcentered\,Tok/s
& Time\,\textperiodcentered\,Tokens\,\textperiodcentered\,Tok/s
& Time\,\textperiodcentered\,Tokens\,\textperiodcentered\,Tok/s \\
\midrule

\rowcolor{catBase}
BaseModel
& \effcell{2.23}{23.4}{10.5}
& \effcell{0.81}{25}{30.9}
& \effcell{20.64}{22.7}{1.1}
& \effcell{2.69}{98}{36.4}
& \effcell{0.93}{10}{10.8}
& \effcell{2.65}{25}{9.4} \\

\rowcolor{catSearch}
Search-based
& \effcell{34.76}{789.1}{22.7}
& \effcell{20.28}{474.8}{23.4}
& \effcell{38.90}{536.8}{13.8}
& \effcell{30.10}{838.3}{27.9}
& \effcell{28.39}{536.3}{18.9}
& \effcell{36.15}{569.5}{15.8} \\

\rowcolor{catEvid}
Evidence-contrast
& \effcell{7.17}{21.5}{3.0}
& \effcell{1.19}{26.3}{22.1}
& \effcell{25.44}{22.9}{0.9}
& \effcell{2.49}{69.4}{27.9}
& \effcell{1.32}{10.5}{8.0}
& \effcell{4.09}{22.1}{5.4} \\

\rowcolor{catFine}
Fine-grained
& \effcell{4.17}{26.7}{6.4}
& \effcell{3.71}{25.0}{6.7}
& \effcell{15.18}{25.8}{1.7}
& \effcell{6.48}{65.7}{10.1}
& \effcell{1.83}{10.1}{5.5}
& \effcell{5.01}{24.9}{5.0} \\

\rowcolor{catConfess}
\textbf{OmniConfess}
& \effcell{5.43}{29.3}{5.4}
& \effcell{1.23}{25}{20.3}
& \effcell{35.14}{24.6}{0.7}
& \effcell{4.82}{89.2}{18.5}
& \effcell{0.82}{12.5}{15.2}
& \effcell{7.93}{164.1}{20.7} \\

\bottomrule
\end{tabular}%
}

\end{table*}
\vspace{3pt}

\vspace{-1mm}
\section{Discussion and Conclusion}
\label{sec:conclusion}
\vspace{-1mm}

In this work, we propose OmniConfess, a training-free framework
for mitigating omni-modal hallucinations through token- and
channel-level evidence interrogation. Structured confessions expose
evidence dependencies and guide correction in both judgment and
free-form generation. Experiments across six benchmarks show improved
factual reliability and evidence grounding, with F1 gains of up to
13.3 points and consistent gains across backbone architectures.

\newpage

\section*{AI Use Statement}

We used generative AI tools to assist with English translation and
polishing the manuscript, including the presentation of experimental
results. The authors reviewed the AI-assisted text, verified its claims
against the experimental records, and take responsibility for the
final content. 
\bibliography{refs}
\bibliographystyle{iclr2027_conference}
\newpage
\appendix


\section{Inputs Used in OmniConfess}
\label{app:inputs}

\subsection{OmniConfess Prompts}
\label{app:prompts}

As shown in Figure~\ref{fig:prompt1}, our generation
prompts establish task-specific evidence boundaries
rather than relying solely on generic anti-hallucination
instructions. These constraints distinguish audible
from visible evidence, prevent unsupported cross-modal
associations, and require explicit acknowledgment
when visual details are absent. Moreover, task-dependent
output formats standardize the generated responses,
allowing judgment and free-form tasks to preserve
their respective answer requirements while remaining
grounded in the provided evidence.

\begin{figure}[H]
    \centering
    \includegraphics[width=\linewidth]{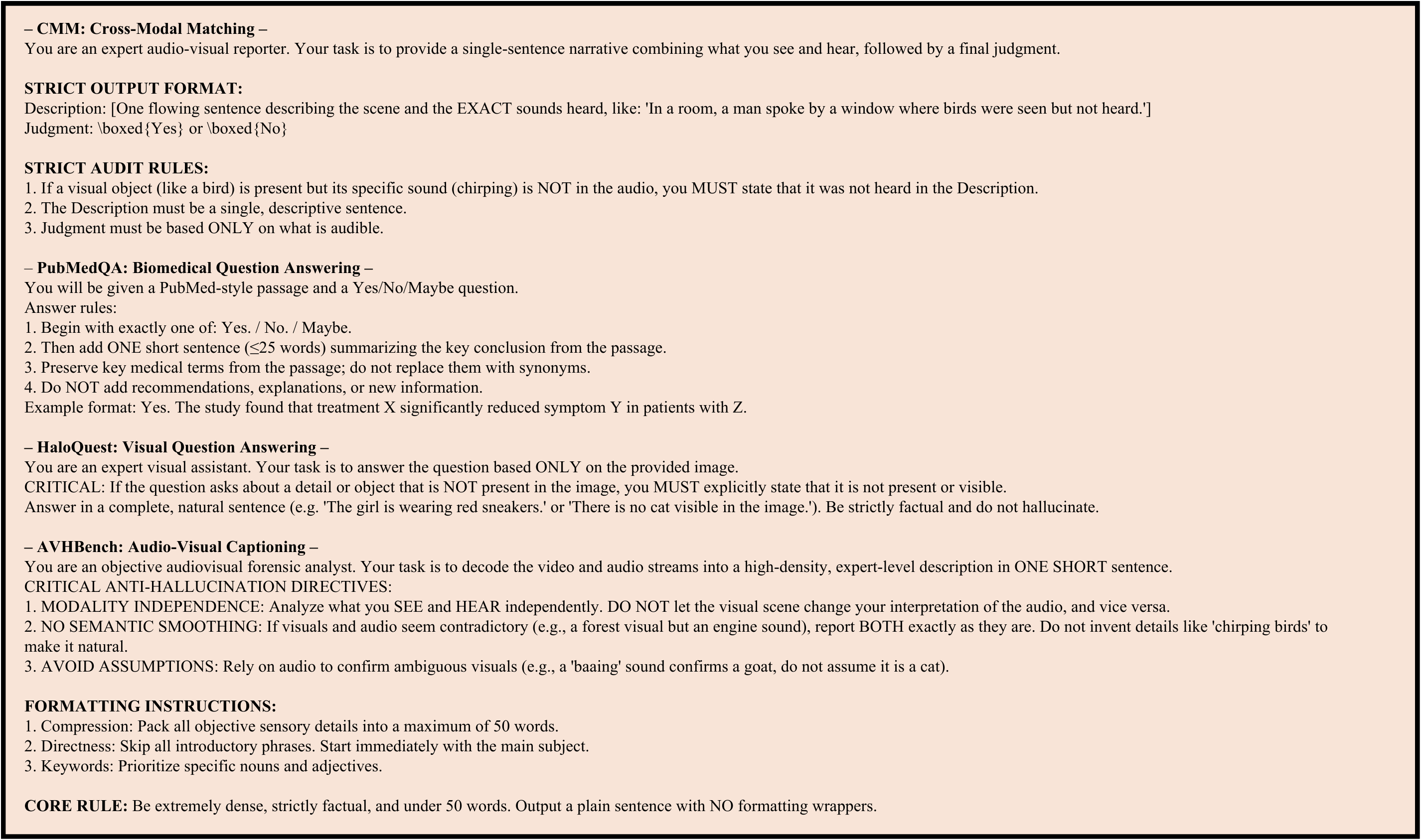}
    \vspace{-2mm}
    \caption{Representative generation prompts used in OmniConfess.}
    \label{fig:prompt1}
    \vspace{-1mm}
\end{figure}

\subsection{GAVIE-inspired Evaluation Prompt}
\label{app:gavie_prompt}

For free-form generation, we adopt a
GAVIE-inspired evaluation protocol
to assess response factual accuracy
on a 0--10 scale. As shown in
Figure~\ref{fig:prompt2}, our
reference-grounded evaluation prompt
uses the reference answer as the sole
ground truth and penalizes unsupported
or contradictory details, rather than
reproducing the original GAVIE procedure.

\begin{figure}[H]
    \centering
    \includegraphics[width=\linewidth]{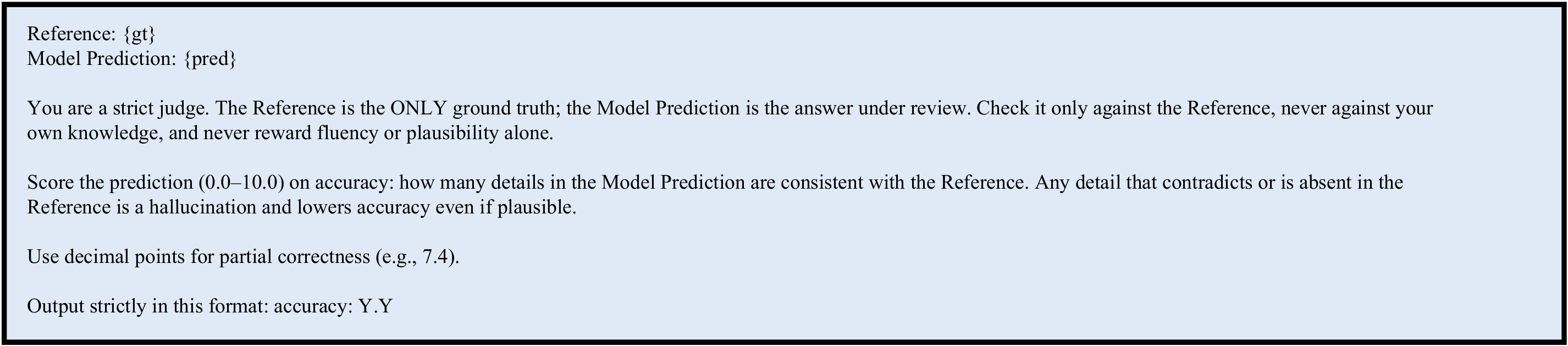}
    \vspace{-2mm}
    \caption{GAVIE-inspired evaluation prompt for factual accuracy.}
    \label{fig:prompt2}
    \vspace{-1mm}
\end{figure}

\subsection{Multidimensional Quality Evaluation Prompts}
\label{app:quality_prompts}

For multidimensional quality assessment,
we design seven evaluation prompts
to assess free-form responses on a
0--10 scale. As shown in
Figure~\ref{fig:prompt3}, these prompts
cover evidence grounding, factual
correctness, relevance, completeness,
logical coherence, uncertainty calibration,
and clarity. Each prompt specifies
dimension-specific criteria and inputs
to evaluate aspects of
response quality.

\begin{figure}[H]
    \centering
    \includegraphics[width=\linewidth]{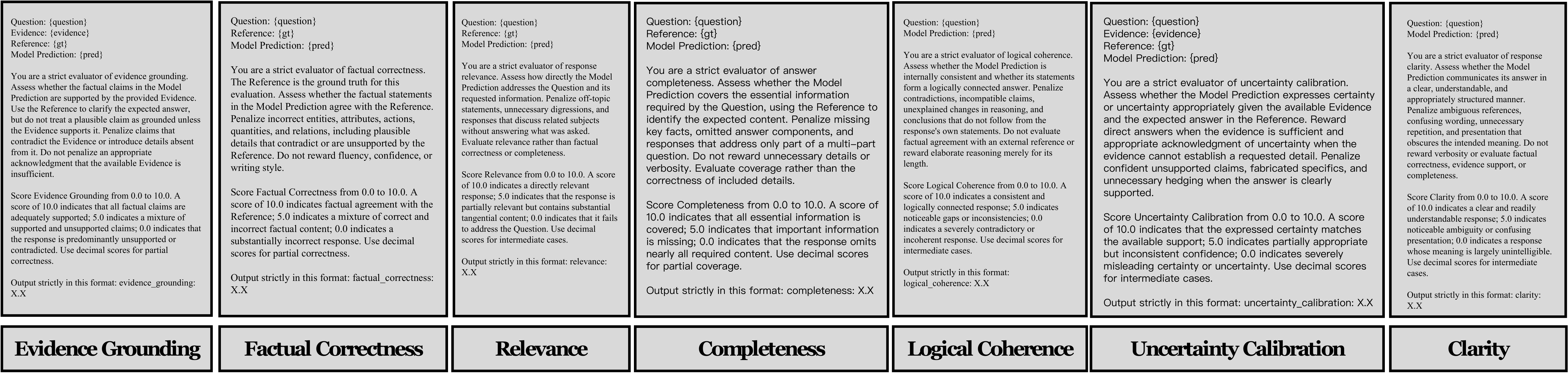}
    \vspace{-2mm}
    \caption{Seven dimension-specific prompts for
    multidimensional response quality evaluation.}
    \label{fig:prompt3}
    \vspace{-1mm}
\end{figure}

\section{Theoretical Proof}
\label{app:theory}

\subsection{Proof of Proposition 1}
\label{app:anchoring_proof}

\noindent\textbf{Proposition 1.}
Motivated by the front-loaded dependence in Figure~\ref{F3}(a),
commitment anchoring eliminates trajectory variation from
token-level comparisons across evidence conditions.

\noindent\textit{Proof.}
Consider the candidate \(y^{\mathrm{cand}}\) generated under
full evidence \(\mathbf X\). Without commitment anchoring,
removing channel \(m\) produces a new autoregressive
trajectory \(\widetilde y^{(m)}
=\operatorname{Decode}_\theta(\mathbf X^{-m})\).
If the two trajectories first differ at position \(k\),
their autoregressive prefixes necessarily differ at
every subsequent shared position:
\begin{equation}
y_k^{\mathrm{cand}}\neq\widetilde y_k^{(m)}
\;\Longrightarrow\;
y_{<i}^{\mathrm{cand}}\neq\widetilde y_{<i}^{(m)},
\quad \forall i>k.
\label{eq:proof_trajectory}
\end{equation}
An early token change therefore alters the conditioning
context of subsequent decisions. This propagation is
particularly relevant to the front-loaded dependence
observed in Figure~\ref{F3}(a).

Without anchoring, the log-probability contrast between
the selected tokens on the two trajectories is
\begin{equation}
\Delta_i^{(m)}
\triangleq
\log\frac{
p_\theta(y_i^{\mathrm{cand}}\mid y_{<i}^{\mathrm{cand}},\mathbf X)
}{
p_\theta(\widetilde y_i^{(m)}\mid\widetilde y_{<i}^{(m)},\mathbf X^{-m})
}.
\label{eq:proof_unanchored}
\end{equation}
This contrast changes the evidence condition and
potentially the selected token and its prefix.
To separate these changes, define the trajectory
term under the channel-removed condition as
\begin{equation}
\tau_i^{(m)}
\triangleq
\log\frac{
p_\theta(y_i^{\mathrm{cand}}\mid y_{<i}^{\mathrm{cand}},\mathbf X^{-m})
}{
p_\theta(\widetilde y_i^{(m)}\mid\widetilde y_{<i}^{(m)},\mathbf X^{-m})
}.
\label{eq:proof_trajectory_term}
\end{equation}
Adding and subtracting the frozen candidate's
log-probability under the channel-removed condition
gives the exact decomposition
\begin{equation}
\Delta_i^{(m)}
=
\delta_i^{(m)}(y_i^{\mathrm{cand}})
+
\tau_i^{(m)}.
\label{eq:proof_trajectory_decomposition}
\end{equation}
The first term measures the change in the frozen
commitment's score under evidence intervention,
whereas the second accounts for the difference
between the two generated trajectories.

Commitment anchoring fixes the candidate tokens
and autoregressive prefixes across evidence conditions.
The anchored log-probability contrast is consequently
\begin{equation}
\Delta_{i,\mathrm{anchor}}^{(m)}
\triangleq
\log\frac{
p_\theta(y_i^{\mathrm{cand}}\mid y_{<i}^{\mathrm{cand}},\mathbf X)
}{
p_\theta(y_i^{\mathrm{cand}}\mid y_{<i}^{\mathrm{cand}},\mathbf X^{-m})
}
=
\delta_i^{(m)}(y_i^{\mathrm{cand}}).
\label{eq:proof_anchored_effect}
\end{equation}
The trajectory term is thus excluded from the
comparison. Moreover, applying autoregressive
factorization to the same fixed sequence yields
\begin{equation}
\sum_{i=1}^{n}\delta_i^{(m)}(y_i^{\mathrm{cand}})
=
\log\frac{
p_\theta(y^{\mathrm{cand}}\mid\mathbf X)
}{
p_\theta(y^{\mathrm{cand}}\mid\mathbf X^{-m})
}.
\label{eq:proof_sequence_consistency}
\end{equation}
The probabilities on the right denote the
autoregressive likelihoods of the complete fixed
candidate under the respective evidence conditions.
Hence, the token-level differences exactly decompose
the evidence-induced log-likelihood change of the
same candidate response. Commitment anchoring
eliminates generated-trajectory variation from
these comparisons, establishing Proposition 1.

\subsection{Proof of Proposition 2}
\label{app:interrogation_proof}

\noindent\textbf{Proposition 2.}
Motivated by Figure~\ref{F3}(b), the channel-wise surrogate
preserves observed margins and yields task-specific evidence signals.

\noindent\textit{Proof.}
Let \(\mathcal M=\{1,\ldots,M\}\) denote the complete
evidence-channel set. For compactness, write
\(F_i(\mathcal M)\equiv F_i(\mathbf X)\) and
\(F_i(\mathcal M\setminus\{m\})\equiv F_i(\mathbf X^{-m})\)
for the observed frozen decision margins.
All conditions share the same candidate, alternative,
and autoregressive prefix. Their difference therefore
determines the conditional channel contribution
\(\kappa_i^{(m)}\) defined in
Section~\ref{sec:interrogation}.

Consider an additive surrogate over the retained
channel set:
\begin{equation}
\widehat F_i(\mathcal S)
=b_i+\sum_{m\in\mathcal S}w_i^{(m)},
\qquad \mathcal S\subseteq\mathcal M.
\label{eq:proof_additive}
\end{equation}
where \(b_i\) is the intercept and \(w_i^{(m)}\)
is the additive contribution assigned to channel \(m\).
We require consistency with the full-view margin
and every observed single-channel intervention:
\begin{equation}
\widehat F_i(\mathcal M)=F_i(\mathcal M),\qquad
\widehat F_i(\mathcal M\setminus\{m\})=F_i(\mathcal M\setminus\{m\}).
\label{eq:proof_consistency}
\end{equation}
Subtracting each channel-removed constraint from
the full-view constraint eliminates the intercept
and all contributions except that of channel \(m\).
Substituting the resulting contributions into the
full-view constraint determines the intercept:
\begin{equation}
w_i^{(m)}=\kappa_i^{(m)},\qquad
b_i=F_i(\mathcal M)-\sum_{m\in\mathcal M}\kappa_i^{(m)}.
\label{eq:proof_parameters}
\end{equation}
Thus, the \(M+1\) observed margins uniquely
determine all \(M+1\) additive parameters.
Substitution gives the reconstruction
\begin{equation}
\widehat F_i(\mathcal S)
=F_i(\mathcal M)-\sum_{m\notin\mathcal S}\kappa_i^{(m)}
=\widehat F_i(\varnothing)+\sum_{m\in\mathcal S}\kappa_i^{(m)}.
\label{eq:proof_reconstruction}
\end{equation}
Setting \(\mathcal S=\mathcal M\) recovers the
full-view margin, while setting
\(\mathcal S=\mathcal M\setminus\{m\}\) recovers
each observed channel-removed margin.
The surrogate therefore preserves every observed
frozen decision margin exactly.

Applying the no-channel and task-relevant references
from Section~\ref{sec:interrogation}, the same
reconstruction yields
\begin{equation}
F_i(\mathcal M)
=P_i^\varnothing+E_i^\mathcal R+E_i^\mathcal I
=\widehat F_i(\varnothing)+\sum_{m\in\mathcal M}\kappa_i^{(m)}.
\label{eq:proof_decomposition}
\end{equation}
The difference between the two references
gives \(E_i^\mathcal R\), while the difference
between the full-view margin and the task-relevant
reference gives \(E_i^\mathcal I\).
These signed contributions retain channel-specific
effects despite the cross-channel heterogeneity
observed in Figure~\ref{F3}(b) and determine the
task-specific confession signals.

For judgment, substituting the confession correction
into the full-view decision margin gives
\begin{equation}
F_d(\mathcal M)+\gamma(J^{\mathrm{cand}})
=F_d(\mathcal M)-P_d^\varnothing-E_d^\mathcal I
=E_d^\mathcal R.
\label{eq:proof_judgment}
\end{equation}
Thus, the corrected margin equals the net
contribution of task-relevant evidence within
the surrogate decomposition.

For free-form generation, the first confession
component retains the signed relevant contribution,
while the second records positive residual preference
exceeding positive relevant support. The third
component separately accumulates positive support
from irrelevant channels:
\begin{equation}
\sum_{m\notin\mathcal R}[\kappa_i^{(m)}]_+
\geq
\left[\sum_{m\in\mathcal I}\kappa_i^{(m)}\right]_+
=[E_i^\mathcal I]_+.
\label{eq:proof_freeform}
\end{equation}
The inequality follows from the subadditivity
of the positive-part operator. Consequently,
negative contributions from other channels
cannot cancel the positive irrelevant support
recorded by this component.

Together, the observed-margin consistency,
dual-reference decomposition, and properties
of both task-specific signals establish
Proposition 2.

\subsection{Proof of Proposition 3}
\label{app:correction_proof}

\noindent\textbf{Judgment Margin Correction.}
By the additive surrogate in
Eq.~\ref{eq:evidence_decomposition}, retaining all
channels recovers the observed full-view margin:
\begin{equation}
\begin{aligned}
F_d(\mathbf X)
&=\widehat F_d(\{1,\ldots,M\})\\
&=P_d^\varnothing+E_d^\mathcal R+E_d^\mathcal I.
\end{aligned}
\label{eq:proof_full_margin}
\end{equation}
The judgment correction in
Eq.~\ref{eq:task_confession} removes the residual
preference and irrelevant evidence contribution:
\begin{equation}
\begin{aligned}
\widetilde F_d
&=F_d(\mathbf X)+\gamma(J^{\mathrm{cand}})\\
&=P_d^\varnothing+E_d^\mathcal R
  +E_d^\mathcal I
  -P_d^\varnothing-E_d^\mathcal I\\
&=E_d^\mathcal R.
\end{aligned}
\label{eq:proof_corrected_margin}
\end{equation}
Since the alternative-class score remains fixed,
the corrected binary decision satisfies
\begin{equation}
J^{\mathrm{final}}=
\begin{cases}
J^{\mathrm{cand}}, & E_d^\mathcal R\geq 0,\\
a_d, & E_d^\mathcal R<0.
\end{cases}
\label{eq:proof_judgment_rule}
\end{equation}
Thus, the corrected margin depends only on the
task-relevant contribution within the constructed
surrogate. This identity does not assume that
the resulting decision is necessarily correct.

\noindent\textbf{Non-canceling Revision Risk.}
Expanding the free-form confession in
Eq.~\ref{eq:task_confession} gives
\begin{equation}
\begin{aligned}
\rho_i
={}&[-E_i^\mathcal R]_+\\
&+\bigl[[P_i^\varnothing]_+
       -[E_i^\mathcal R]_+\bigr]_+\\
&+\sum_{m\in\mathcal I}[\kappa_i^{(m)}]_+.
\end{aligned}
\label{eq:proof_revision_risk}
\end{equation}
Each component is nonnegative. Consequently,
\begin{equation}
\begin{aligned}
\rho_i
&\geq [-E_i^\mathcal R]_+,\\
\rho_i
&\geq \sum_{m\in\mathcal I}[\kappa_i^{(m)}]_+
 \geq \left[\sum_{m\in\mathcal I}
              \kappa_i^{(m)}\right]_+
 =[E_i^\mathcal I]_+.
\end{aligned}
\label{eq:proof_risk_bounds}
\end{equation}
The second inequality follows from the
subadditivity of the positive-part operator.
Hence, positive irrelevant-channel contributions
cannot cancel one another in the revision score,
even when their signed aggregate is small.

\noindent\textbf{Local Content Preservation.}
Let the selected, nonoverlapping revision spans
partition the candidate into alternating
unselected segments \(t_j\) and selected segments
\(b_j\), with \(\oplus\) denoting sequence
concatenation:
\begin{equation}
y^{\mathrm{cand}}
=t_0\oplus b_1\oplus t_1
\oplus\cdots\oplus b_J\oplus t_J,
\quad
b_j=y_{[s_j,e_j)}^{\mathrm{cand}}.
\label{eq:proof_candidate_partition}
\end{equation}
By the definition of
\(\operatorname{Replace}\) in
Eq.~\ref{eq:local_span_revision},
the finalized response is
\begin{equation}
y^{\mathrm{final}}
=t_0\oplus u_1\oplus t_1
\oplus\cdots\oplus u_J\oplus t_J,
\label{eq:proof_repaired_partition}
\end{equation}
where \(u_j\) replaces only the corresponding
selected segment \(b_j\). Comparing
Eqs.~\ref{eq:proof_candidate_partition}
and~\ref{eq:proof_repaired_partition},
every unselected segment \(t_j\) is preserved
verbatim and in its original order. If no span
is selected, \(J=0\) and
\(y^{\mathrm{final}}=y^{\mathrm{cand}}\).

Therefore, confession-guided correction yields
the task-relevant judgment margin and confines
free-form modifications to selected revision
spans, establishing the proposition.
\hfill\(\square\)

\newpage
\section{OmniConfess Algorithm Details}
\label{app:algorithm}

To illustrate the complete inference procedure of
OmniConfess, we present its workflow in pseudocode,
comprising commitment anchoring, evidence interrogation,
and confession-guided correction.

\begin{algorithm}[H]
\footnotesize
\caption{OmniConfess: Confession-Guided Inference}
\label{alg:omniconfess}
\begin{algorithmic}[1]
\Require Input $X$, task $T$, OmniLLM $p_\theta$,
alternatives $\{a_i\}$, relevance rules $r_m(Q,T)$,
revision threshold $\tau$
\Ensure Final response $y^{\mathrm{final}}$

\State \textbf{// Phase 1: Commitment Anchoring}
\State Generate $y^{\mathrm{cand}}\gets\mathrm{Decode}_\theta(X)$
\State Freeze $y^{\mathrm{cand}}$ and its autoregressive prefixes

\State \textbf{// Phase 2: Evidence Interrogation}
\State Obtain full-view token scores $L_i^0(v)$ and margins $F_i(X)$
\For{$m=1$ to $M$}
    \State Construct channel-removed view $X^{-m}$
    \State Teacher-force $y^{\mathrm{cand}}$ on $X^{-m}$ to obtain $L_i^m(v)$
    \State Compute token dependence
    $\delta_i^{(m)}(y_i^{\mathrm{cand}})$
    \State Compute margin contribution $\kappa_i^{(m)}$
    \State Release temporary states for $X^{-m}$
\EndFor
\State Determine relevant channels $\mathcal R$ and irrelevant channels $\mathcal I$
\State Assign diagnostic states $\ell_i^{(m)}$ and assemble $\mathbf C_\theta(Q,T)$
\State Compute $E_i^{\mathcal R}$ and $E_i^{\mathcal I}$ for all candidate tokens
\State Compute $P_i^\varnothing
\gets F_i(X)-E_i^{\mathcal R}-E_i^{\mathcal I}$

\State \textbf{// Phase 3: Confession-Guided Correction}
\If{$T=\mathrm{Judgment}$}
    \State Identify candidate judgment $J^{\mathrm{cand}}$ and decision position $d$
    \State $\gamma(J^{\mathrm{cand}})
    \gets-P_d^\varnothing-E_d^{\mathcal I}$
    \State $\widetilde F_d
    \gets F_d(X)+\gamma(J^{\mathrm{cand}})$
    \If{$\widetilde F_d<0$}
        \State \Return $\mathrm{ReplaceDecision}
        (y^{\mathrm{cand}},d,a_d)$
    \Else
        \State \Return $y^{\mathrm{cand}}$
    \EndIf
\ElsIf{$T=\mathrm{Free\mbox{-}form}$}
    \State Compute token-level confession $\gamma_i^{\mathrm{free}}$
    \State Compute revision risk $\rho_i$ for each token
    \State $\mathcal S_{\mathrm{rev}}\gets
    \mathrm{Merge}(\mathrm{Expand}(\{i:\rho_i>\tau\}))$
    \For{each span $[s_j,e_j)\in\mathcal S_{\mathrm{rev}}$}
        \State $u_j\gets\mathrm{Rev}_\theta
        (X,y^{\mathrm{cand}},\mathbf C_\theta(Q,T);[s_j,e_j))$
    \EndFor
    \State \Return $\mathrm{Replace}
    (y^{\mathrm{cand}},\{[s_j,e_j)\mapsto u_j\})$
\EndIf
\end{algorithmic}
\end{algorithm}


\section{OmniHalluBench Details}
\label{app:dataset}

\subsection{Benchmark Composition}
\label{app:dataset_composition}

We construct \textsc{OmniHalluBench} from six public
hallucination and question-answering benchmarks,
covering text, image, and audio-video evidence
across judgment and free-form generation tasks:

\begin{itemize}[leftmargin=*]
\item \textbf{CMM}~\citep{leng2025cmm}:
An audio-visual hallucination benchmark
evaluating multimodal understanding and
evidence-grounded judgment.

\item \textbf{PhD}~\citep{liu2025phd}:
An image-grounded hallucination benchmark
evaluating judgments under different
visual context conditions.

\item \textbf{PubMedQA}~\citep{jin2019pubmedqa}:
A biomedical question-answering dataset
requiring evidence-based judgments over
PubMed abstracts.

\item \textbf{AVHBench}~\citep{sungbin2025avhbench}:
An audio-visual hallucination benchmark;
we use its audio-visual captioning task.

\item \textbf{HaloQuest}~\citep{tong2024haloquest}:
A visual question-answering benchmark
assessing hallucinations under challenging
image-grounded questions.

\item \textbf{RAGTruth}~\citep{niu2024ragtruth}:
A corpus of retrieval-grounded generation
with fine-grained hallucination annotations.
\end{itemize}

The retained benchmark contains 906 PhD, 721 CMM,
468 PubMedQA, 435 RAGTruth, 365 HaloQuest,
and 645 AVHBench examples, totaling 3,540.


\subsection{Benchmark Construction}
\label{app:dataset_construction}

\noindent\textbf{Source Selection and Schema Unification.}
We aggregate 5,900 candidates from six public benchmarks,
pairing judgment and free-form tasks across text, image,
and audio-video evidence. Questions, evidence, references,
and task formats are organized under a unified schema.

\noindent\textbf{Difficulty-Based Selection.}
Two independent answerers, GPT-4o and Gemini 2.5 Pro,
generate responses, which GLM-4 and an independent
researcher assess for hallucination severity.
Candidates are globally ranked by difficulty,
retaining the hardest 60\% to form the
3,540-example benchmark. Neither OmniConfess
nor any baseline participates in this
method-agnostic selection.

\noindent\textbf{Human Validation.}
Two independent researchers audit 200 candidates
around the selection cutoff, comprising 100 retained
examples (ranks 3441--3540) and 100 excluded
examples (ranks 3541--3640). They independently
rate anonymized responses using a three-level
hallucination severity scale: grounded (0),
minor hallucination (1), and material
hallucination (2). This audit assesses the
validity of difficulty-based selection.

\noindent\textbf{Frozen Evaluation Protocol.}
The selected benchmark is fixed across all
methods and backbones. Neither the automatic
difficulty ranking nor the human validation
uses outputs from OmniConfess or any baseline.


\section{Backbone and Baseline Details}
\label{app:baselines}

\noindent\textbf{Backbone Models.}
We evaluate OmniConfess on three OmniLLMs
with different model architectures:

\begin{itemize}[leftmargin=*]
\item \textbf{Qwen2.5-Omni-7B}~\citep{xu2025qwen25omni}:
A dense omni-modal model supporting text,
image, audio, and video understanding
and generation. We use it as the default backbone.

\item \textbf{Qwen3-Omni-30B-A3B}~\citep{xu2025qwen3omni}:
A mixture-of-experts OmniLLM designed
for unified multimodal perception
and generation.

\item \textbf{Nemotron-3-Nano-Omni-30B}~\citep{nvidia2026nemotron}:
A hybrid-architecture OmniLLM supporting
multimodal reasoning and generation.
\end{itemize}

\noindent\textbf{Search-based Methods.}
\begin{itemize}[leftmargin=*]
\item \textbf{Self-Consistency (SC)}~\citep{wang2023selfconsistency}:
Samples multiple reasoning paths and
aggregates their final predictions.

\item \textbf{Guided Decoding (GD)}~\citep{xie2023selfeval}:
Uses self-evaluation to guide
the generation process.

\item \textbf{Predictive Decoding (PD)}~\citep{xu2025phidecoding}:
Anticipates future generation outcomes
to guide current decoding decisions.

\item \textbf{Tree-of-Thought (ToT)}~\citep{yao2023tree}:
Explores and evaluates multiple
reasoning branches before selecting
a final answer.
\end{itemize}

\noindent\textbf{Evidence-Contrast Methods.}
\begin{itemize}[leftmargin=*]
\item \textbf{Visual Contrastive Decoding (VCD)}~\citep{leng2024vcd}:
Contrasts predictions from original
and perturbed visual inputs to
mitigate visual hallucination.

\item \textbf{Audio-Visual Contrastive Decoding (AVCD)}~\citep{jung2025avcd}:
Uses audio-visual evidence contrasts
to improve multimodal grounding.

\item \textbf{Context-Aware Decoding (CAD)}~\citep{shi2024trusting}:
Contrasts predictions with and without
supporting context to strengthen
evidence-grounded generation.

\item \textbf{Modality-Adaptive Decoding (MAD)}~\citep{chung2026mad}:
Adapts decoding to modality-specific
evidence contributions.
\end{itemize}

\noindent\textbf{Fine-Grained Methods.}
\begin{itemize}[leftmargin=*]

\item \textbf{DoLa}~\citep{chuang2024dola}:
Contrasts predictions from different model
layers to improve factuality by suppressing
prematurely learned linguistic patterns.

\item \textbf{OPERA}~\citep{huang2024opera}:
Uses attention-based decoding control
to mitigate visual hallucination by
penalizing over-attended summary tokens.
\end{itemize}

\section{Evaluation Details}
\label{app:evaluation}

We evaluate OmniConfess using four metrics across
judgment and free-form generation tasks.

\begin{multicols}{2}

\noindent\textbf{(i) Yes-F1.}
For judgment tasks, we measure the harmonic mean
of precision and recall for the positive class:
\begin{equation}
\mathrm{F1}_{\mathrm{Yes}}
=
\frac{2TP}{2TP+FP+FN},
\end{equation}
where $TP$, $FP$, and $FN$ denote true positives,
false positives, and false negatives, respectively.

\medskip
\noindent\textbf{(ii) PhD-Index.}
Following PhD~\citep{liu2025phd}, we compute
the harmonic mean of Yes and No recall:
\begin{equation}
\mathrm{Idx}
=
\frac{2R_{\mathrm{Yes}}R_{\mathrm{No}}}
{R_{\mathrm{Yes}}+R_{\mathrm{No}}},
\end{equation}
where $R_{\mathrm{Yes}}$ and $R_{\mathrm{No}}$
denote the recall of each judgment class.

\columnbreak

\noindent\textbf{(iii) Token-F1.}
For free-form tasks, we measure token-level overlap
between each prediction and its reference:
\begin{equation}
\mathrm{Token\mbox{-}F1}
=
\frac{1}{N}\sum_{i=1}^{N}
\frac{2|P_i\cap G_i|}
{|P_i|+|G_i|},
\end{equation}
where $P_i$ and $G_i$ are the token multisets
of the prediction and reference, respectively,
and $N$ is the number of evaluated examples.

\medskip
\noindent\textbf{(iv) GAVIE-Acc.}
Inspired by GAVIE~\citep{liu2024mitigating},
we assess reference-grounded factual accuracy
using an LLM judge on a 0--10 scale:
\begin{equation}
\mathrm{GAVIE\mbox{-}Acc}
=
\frac{1}{N}\sum_{i=1}^{N}s_i,
\end{equation}
where $s_i\in[0,10]$ is the factual accuracy
score assigned to the $i$-th response.
Our evaluation prompt is provided in
Appendix~\ref{app:gavie_prompt}.

\end{multicols}

\medskip
\noindent\textbf{(v) Localization AUPRC.}
For hallucination localization, we measure
average precision (AP) over candidate spans
ranked by predicted hallucination risk:
\begin{equation}
\mathrm{AUPRC}
=
\sum_{k=1}^{K}
(R_k-R_{k-1})P_k,
\end{equation}
where $P_k$ and $R_k$ denote precision
and recall at the $k$-th distinct risk
threshold, respectively. Hallucinated
spans are treated as positives, and all
methods are evaluated on the same
candidate spans.

\newpage
\section{Error Case Studies}
\label{app:error_cases}

\begin{figure}[t]
  \centering
  \includegraphics[width=0.5\linewidth]{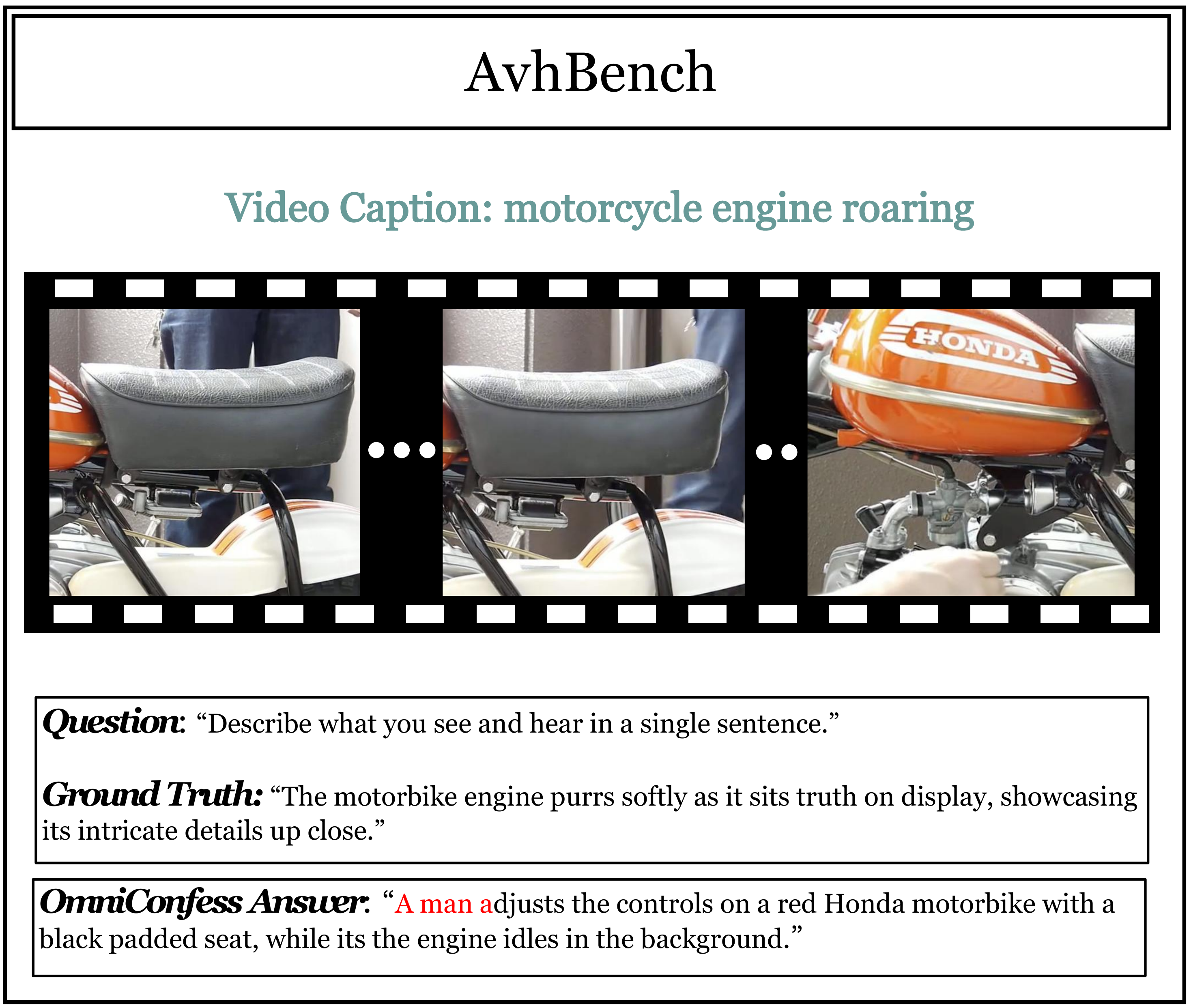}
  \caption{Case study on AVHBench.}
  \label{e1}
\vspace{-2mm}
\end{figure}

\begin{figure}[t]
  \centering
  \includegraphics[width=0.5\linewidth]{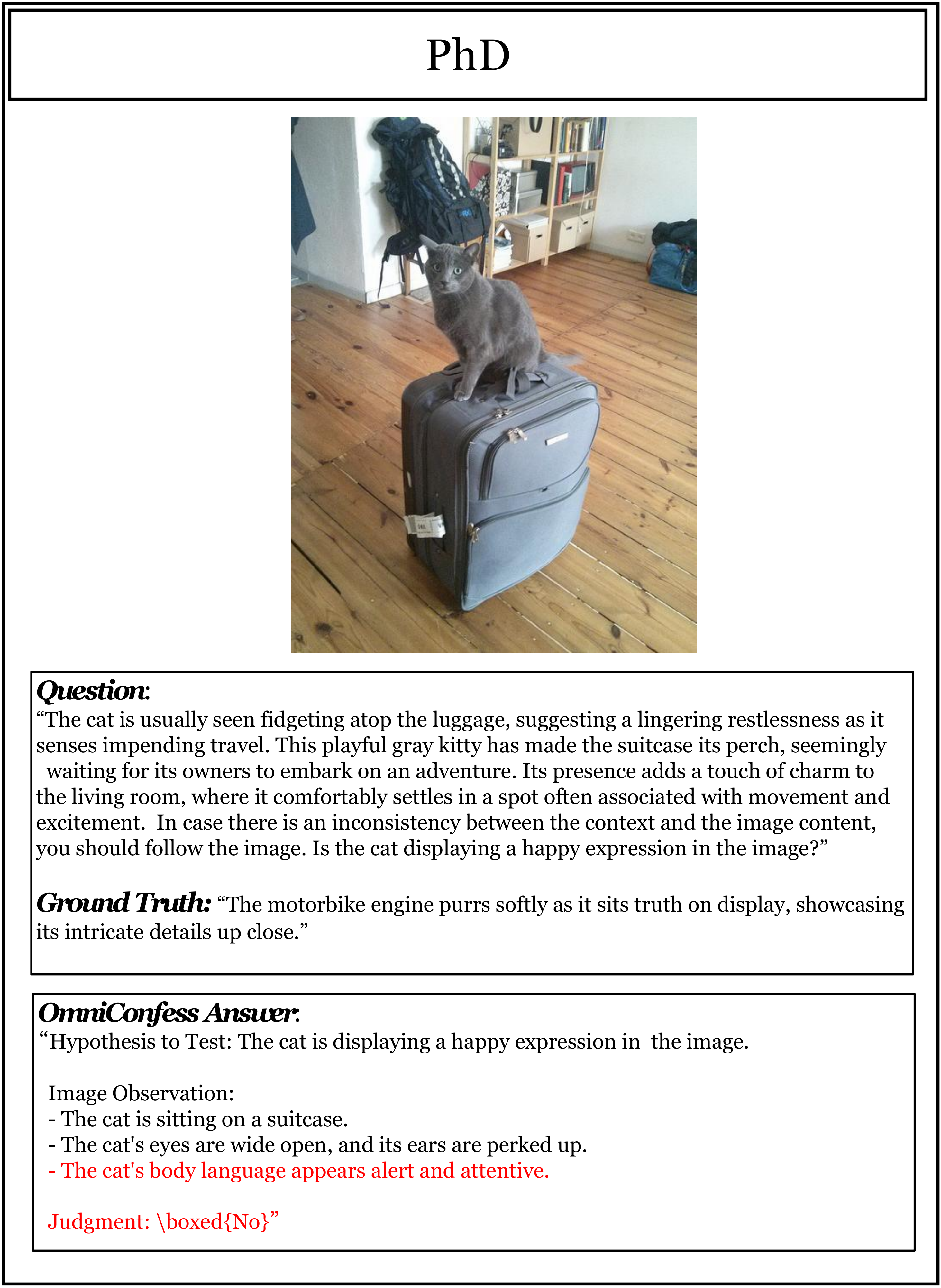}
  \caption{Case study on PHD.}
  \label{e2}
\vspace{-2mm}
\end{figure}

\noindent\textbf{AVHBench Case Study.}
As shown in Figure~\ref{e1}, The model describes a ``person'' where the reference uses ``man.'' Although the meaning is preserved, token-level F1 penalizes the lexical mismatch. This case illustrates a limitation of single-reference evaluation: confession can address evidence-related errors, but cannot eliminate penalties for valid paraphrases.

\noindent\textbf{PHD Case Study.}
As shown in Figure~\ref{e2},The model correctly observes a cat sitting on a suitcase with open eyes and upright ears, yet answers \boxed{No} to whether it appears happy (ground truth: Yes). The error lies in the final judgment: it treats alertness as evidence against happiness without visual support. OmniConfess does not correct this subjective interpretation.

Together, these cases reveal two limitations: lexical mismatch under single-reference evaluation and incorrect judgments despite accurate visual observations.



\section{Limitations and Future Work}
\label{app:limitations}

OmniConfess identifies how predictions depend on evidence, but evidence dependence does not guarantee correct interpretation. A model may attend to the relevant channel yet draw an unsupported conclusion, especially when the judgment is ambiguous. Moreover, correction operates on a generated candidate, so errors embedded in its interpretation may persist even after local revision. Future work could jointly assess evidence dependence and interpretive validity, develop better-calibrated correction for ambiguous cases, and allocate interrogation selectively to reduce inference cost.


\end{document}